%% file: main.tex
\PassOptionsToPackage{colorlinks=true,linkcolor=linkColor,citecolor=linkColor,filecolor=linkColor,urlcolor=linkColor}{hyperref}
\documentclass{article}

\pdfoutput=1

\input{myheaders.tex}

\usepackage[dvipsnames]{xcolor}         % colors
\definecolor{linkColor}{rgb}{0.18,0.39,0.62}
\usepackage[colorlinks=true,linkcolor=linkColor,citecolor=linkColor,filecolor=linkColor,urlcolor=linkColor]{hyperref}       % \usepackage{url}
\usepackage{array}
\usepackage{fullpage}
\usepackage{adjustbox}
\usepackage{lineno}
\usepackage{enumitem}
\usepackage{cancel}
\usepackage[skins]{tcolorbox}
\tcbuselibrary{breakable}
\usepackage[utf8]{inputenc}

\usepackage{subcaption}
\usepackage{wrapfig}
\usepackage{fontawesome5}

\usepackage[table]{xcolor}
\definecolor{myblue}{RGB}{236, 248, 233}

\newcommand\samethanks[1][\value{footnote}]{\footnotemark[#1]}

\title{Multi-Turn On-Policy Distillation with Prefix Replay}
\author{Baohao Liao\thanks{Equal Contribution.} \thanks{Work done during an internship at
Microsoft.} $^1$$^2$\quad Hanze Dong\samethanks[1] $^1$\quad Christof Monz$^2$\quad Xinxing Xu$^1$\quad Li Dong$^1$\quad Furu Wei$^1$\\\\$^1$Microsoft Research\quad $^2$University of Amsterdam
}
\date{\small \texttt{\href{https://aka.ms/GeneralAI}{https://aka.ms/GeneralAI}}}
\begin{document}

\maketitle

\begin{abstract}
\input{sections/abs}
\end{abstract}

\vspace{0mm}
\begin{figure}[H]
    \centering
    \includegraphics[width=0.99\linewidth]{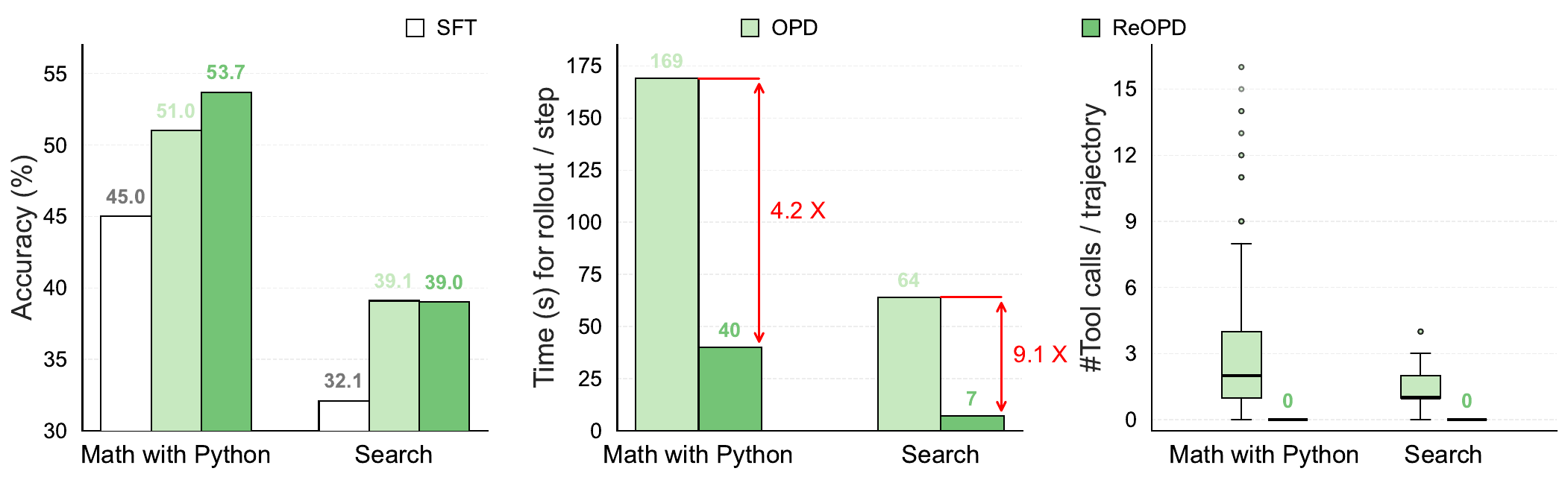}
    \vspace{-0mm}
  \caption[]{\textbf{ReOPD keeps the benefits of on-policy distillation while removing
  environment interaction.} ReOPD
  matches or improves OPD accuracy, but trains much faster per step and eliminates
  tool calls during training by replaying teacher-recorded prefixes instead of
  executing fresh environment rollouts. The student/teacher models are Qwen3-4B-Instruct-2507 and Qwen3-8B, respectively.}
  \label{fig:overall}
  \vspace{0mm}
\end{figure}

\newpage
\input{sections/intro}
\input{sections/related}
\input{sections/formulation}
\input{sections/aopd}

\input{sections/exp}

\input{sections/conclusion}

\bibliography{myrefs}
\bibliographystyle{apalike}

\end{document}

%% file: myheaders.tex
\usepackage{multirow}
\usepackage{threeparttable}
\usepackage{makecell}
\usepackage{booktabs} 
\usepackage{amsmath,amsfonts,bm}
\usepackage{enumitem}

\def\eqref#1{Equation~(\ref{#1})}
\def\plaineqref#1{\ref{#1}}
\def\1{\bm{1}}

\RequirePackage{amsmath}
\RequirePackage{amssymb}
\RequirePackage{amsthm}
\RequirePackage{bm} 
\RequirePackage{url}
\usepackage{multirow}
\usepackage{natbib}
\usepackage{graphicx}
\usepackage{subcaption}
\usepackage{makecell}
\usepackage{booktabs}
\usepackage{array}
\usepackage{url}
\usepackage{algorithm}
\usepackage{mathtools}

\usepackage{algpseudocode}

\usepackage{dsfont}

\usepackage[OT1]{fontenc}
\usepackage{natbib}

\usepackage{mathrsfs}

\usepackage{xcolor}
\usepackage{hyperref}
\usepackage{float}  % 放在导言区
\usepackage{placeins}  % \FloatBarrier to keep floats within their section
\usepackage{colortbl}
\def\##1\#{\begin{align}#1\end{align}}
\def\$#1\${\begin{align*}#1\end{align*}}

\usepackage{xcolor}

\definecolor{red1}{HTML}{f47983}
\definecolor{blue1}{HTML}{3eede7}
\definecolor{yellow1}{HTML}{f5dd6f}
\newtheorem{assumption}{Assumption}

\newtheorem{proposition}{Proposition}

%% file: sections/abs.tex
We study on-policy distillation (OPD) for agentic tasks, where an LLM agent interacts with an environment over multiple turns and a student imitates a teacher over these multi-turn interaction histories. Fully online OPD is costly because each update requires fresh student rollouts through the environment and teacher queries at visited histories. We propose Replayed-Prefix On-Policy Distillation (ReOPD), an off-environment alternative that reuses pre-collected teacher trajectories as replayed prefixes: the student acts at selected steps, while the teacher provides dense per-step supervision without executing new environment interactions. We show that multi-turn OPD introduces a prefix trap: making histories more student-on-policy improves relevance to the student, but can query the teacher on histories where its target is unreliable. This creates a two-sided distribution shift between student occupancy and teacher reliability. ReOPD addresses this by treating multi-turn OPD as a reliability-aware prefix distribution design and implements it with a simple step-decaying sampling schedule that emphasizes early, lower-shift prefixes. Across mathematical reasoning with Python and search environments over multiple teacher and student model scales, ReOPD preserves or improves OPD-level accuracy, uses zero tool calls during student training, and is at least 4$\times$ faster per rollout than OPD. ReOPD therefore turns expensive agent-environment interaction into a reusable offline resource, enabling scalable distillation across tools, tasks, and environments.

\begin{center}
{\large \faGlobe}\,\textbf{ Project Page: }\href{https://baohaoliao.github.io/ReOPD/}{\texttt{baohaoliao.github.io/ReOPD}} \\ \quad
{\large \faGithub}\,\textbf{ Code: }\href{https://github.com/baohaoliao/ReOPD}{\texttt{baohaoliao/ReOPD}} \quad\quad\quad\quad
\raisebox{-0.18em}{\includegraphics[height=1.1em]{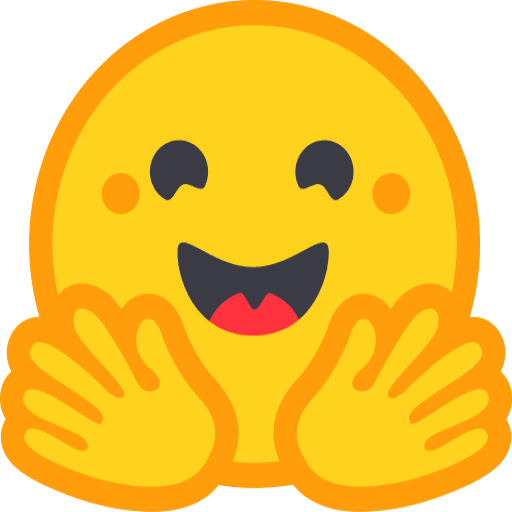}}\,\textbf{ Model \& Data: }\href{https://huggingface.co/collections/baohao/reopd}{\texttt{baohao/reopd}}
\end{center}

%% file: sections/intro.tex
\section{Introduction}

\vspace{-0.em}
\begin{figure}[t]
    \centering
    \includegraphics[width=0.99\linewidth]{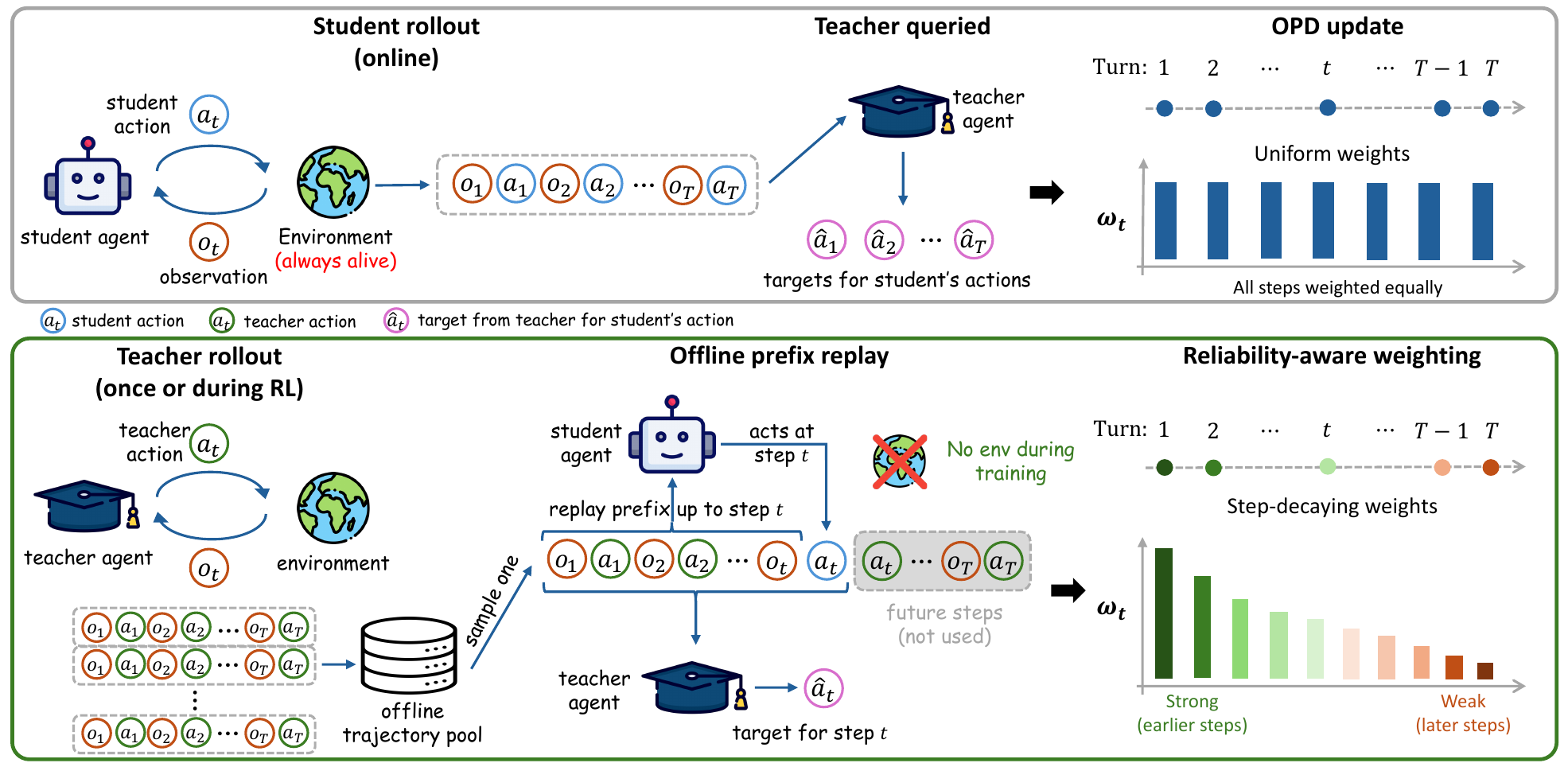}
  \caption[]{\textbf{Comparison between OPD and ReOPD for agentic tasks}. \textbf{Up:} OPD with online environment. The environment is always alive during training. And all steps equally contribute to the loss. \textbf{Down:} ReOPD with offline environment. The environment is only alive for collecting the teacher's trajectories, which can happen during the training of the teacher agent by using RL, like GRPO. Afterwards, the environment is not needed for the training of the student agent. The earlier steps contribute more to the loss.}
  \label{fig:overview}
  \vspace{-0mm}
\end{figure}

\vspace{-0.em}
\begin{figure}[t]
    \centering
    \includegraphics[width=0.99\linewidth]{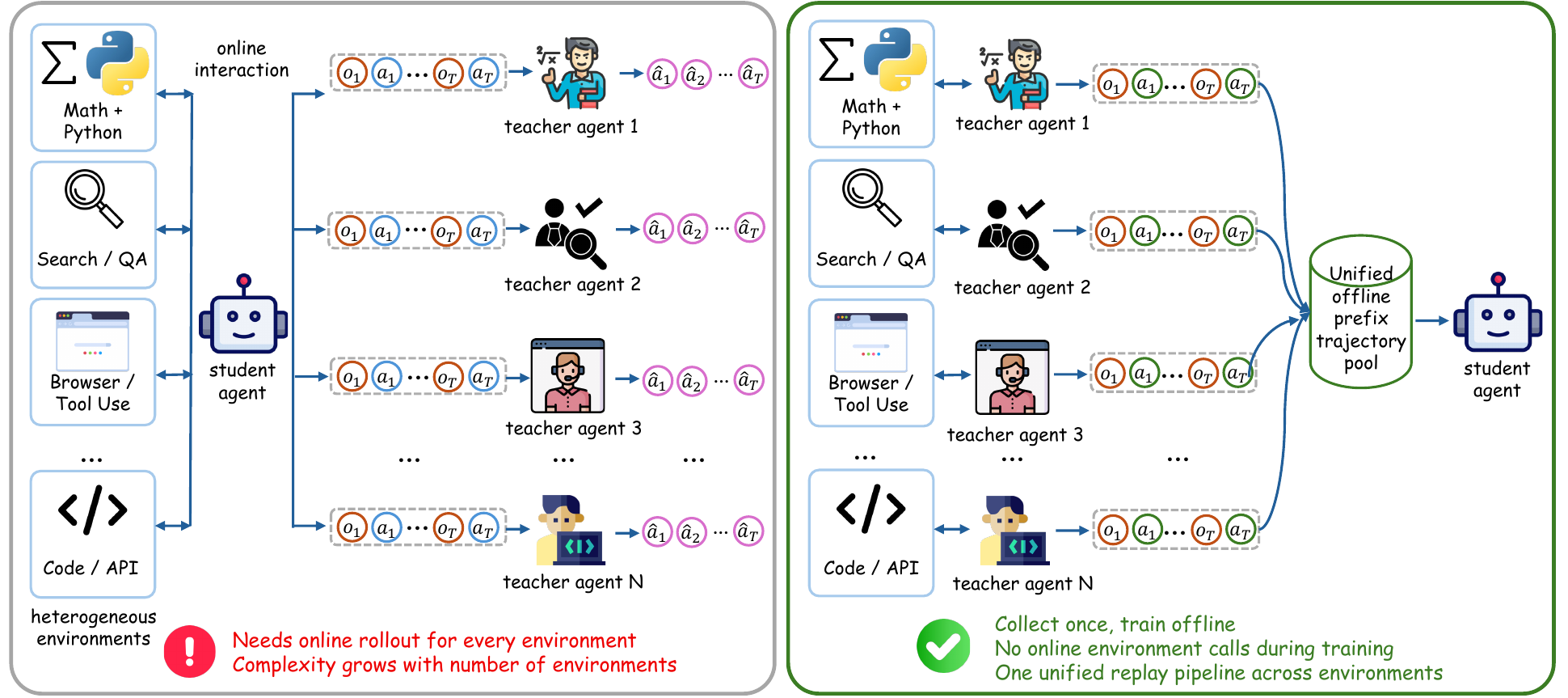}
  \caption[]{\textbf{Training a shared student agent on multiple heterogeneous environments with OPD and ReOPD.} \textbf{Left:} With an increasing number of environments, the operational complexity grows for OPD due to the heavy deployment of the environments. \textbf{Right:} ReOPD doesn't require the deployment of all environments at the same time. The teacher's trajectories can be collected separately for different environments, and then be merged into a unified pool for the later training of the student agent without any online environment.}
  \label{fig:overview_multi_envs}
  \vspace{-0mm}
\end{figure}

Reinforcement learning (RL) has become a central tool for eliciting strong
reasoning and agentic behavior from large language models (LLMs). Policy-gradient
methods such as PPO~\citep{schulman2017proximal} and GRPO~\citep{shao2024deepseekmath},
together with verifiable rewards, drive much of the recent progress on
mathematical reasoning and tool use~\citep{deepseekai2025deepseekr1incentivizingreasoningcapability,yu2025dapo}.
A defining property of RL is that it is \emph{on-policy}: the model is trained on
trajectories sampled from itself, so it learns to recover from its own mistakes
rather than from mistakes made by some other distribution. However, RL supervision
is sparse: a scalar reward conveys only a few bits per episode, regardless of how
many tokens were generated~\citep{thinkingmachines2025onpolicy}, which makes it
sample-inefficient, particularly for small models.

Knowledge distillation offers a much denser learning signal by training a student
to imitate a stronger teacher~\citep{hinton2015distilling,kim2016sequence}. The
standard recipe is \emph{off-policy}: the student is supervised on teacher-generated
trajectories via sequence-level distillation or supervised fine-tuning. While
data-efficient, off-policy distillation trains the student on prefixes visited by
the \emph{teacher}, not by the student itself. When the student later makes an early
mistake the teacher would not, it drifts into prefixes never seen during training,
and errors compound over the sequence~\citep{ross2011reduction,ranzato2016sequence}.
\emph{On-policy distillation} (OPD) addresses this by sampling prefixes from the
student while still distilling against the teacher's per-token
distribution~\citep{gu2024minillm,agarwal2024gkd,thinkingmachines2025onpolicy},
combining the relevance of on-policy data with the density of teacher supervision.

These ideas were largely developed for single-turn generation. Modern LLMs,
however, increasingly tackle \emph{agentic tasks}, operating as \emph{agents} that
interact with external environments over multiple turns, interleaving actions such as code execution,
retrieval, or search engine calls with the environment
feedback~\citep{yao2022react,schick2023toolformer,nakano2021webgpt,jin2025search}.
Coupling this multi-turn structure with teacher distillation gives \emph{multi-turn OPD}:
at each turn, the student rolls into a history, and the teacher supplies an
improvement target. Realized online, this is expensive -- every update must roll the student through the environment and query the teacher afresh at each visited
history, incurring environment-interaction and teacher-inference cost at every
step. A far more efficient alternative -- \emph{prefix replay} -- reuses trajectories the teacher has already
produced (Figure~\ref{fig:overview}): the student is rolled in on a teacher-recorded prefix and generates its
own action at the evaluated step, the prefix and its observations are taken from the
trace instead of executing actions, and the teacher supplies the per-step target.
The student thus stays on-policy at the supervised step -- but not along the
teacher-forced prefix itself -- while the environment is never queried; ReOPD therefore
trades a fully on-policy roll-in for reusability, a tension we make precise as a
two-sided distribution shift in Section~\ref{sec:formulation}. Crucially, such trajectories are not an added cost: when the teacher
is itself trained with RL, its on-policy rollouts are already collected during
training, so the prefix pool comes for free and distillation requires no extra
environment interaction. This decoupling compounds when a single student must learn
from several heterogeneous environments: instead of deploying all of them
simultaneously, ReOPD collects each teacher's trajectories separately and merges them
into one offline pool for student training (Figure~\ref{fig:overview_multi_envs}).
Figure~\ref{fig:overall} illustrates the payoff: ReOPD retains the
accuracy benefits of OPD while sharply reducing per-step training time and tool
calls. This motivates the question we study:
\emph{How should multi-turn OPD be done when the environment is
replaced by replayed teacher prefixes rather than fresh online rollouts?} In
this setting, prefixes act as implicit states and multi-turn generation induces
an implicit trajectory distribution over them.

Training offline in this way does not remove the temporal structure of multi-turn
learning, and we identify a \emph{prefix trap} with two layers. The first is
\emph{temporal}: prefix errors compound across steps, so the problem stays
sequential. The second is \emph{distributional} -- a \emph{two-sided shift}:
Making prefixes fully student-on-policy keeps them relevant to the student (a
student occupancy shift) but can query the teacher where its condition is no
longer a trustworthy target (a teacher reliability shift). A simple decomposition
bounds the gap to an ideal interactive objective by exactly these two terms.
Multi-turn OPD is therefore not just ``make distillation on-policy'' but
\emph{reliability-aware prefix distribution design}: choosing, at each step, an
effective prefix distribution between student and teacher occupancies; a concrete
sampling or weighting schedule is an implementation.

This view predicts a regime-aware outcome that our experiments confirm: OPD is already
strong when the teacher stays reliable on student-induced histories, whereas leaning on teacher trajectories
and down-weighting high-shift late steps wins when the teacher--student gap is large, even though it is less
on-policy.

\paragraph{Contributions.} Our contributions are as follows.
\begin{itemize}[leftmargin=1.5em,itemsep=2pt,topsep=2pt]
  \item We study \emph{efficient} Replayed-Prefix On-Policy Distillation that reuses a
  fixed pool of offline teacher trajectories, and surface the \emph{prefix trap},
  separating its temporal (compounding-error) and distributional (two-sided shift)
  layers.
  \item We formulate multi-turn OPD as a two-sided distribution-shift problem and
  derive a bound that decomposes the objective gap into a student
  occupancy-mismatch term and a teacher reliability term, showing why fully
  student-on-policy distillation is not automatically optimal.
  \item We recast multi-turn OPD as a reliability-aware prefix distribution design, with
  a geometric bridge between student and teacher occupancies, and show that a simple
  step-decay schedule -- applied by sampling prefixes -- is a practical implementation
  of the resulting effective distribution.
  \item We validate the resulting method, \emph{Replayed-Prefix On-Policy Distillation} (ReOPD), on mathematical reasoning and search environments,
  across single- and multi-environment settings and teacher and student models of varying
  scale. ReOPD improves over off-policy distillation (SFT) where reported; relative to
  strong on-policy OPD, and consistent with our two-regime analysis, it improves on
  mathematical reasoning when the teacher-student gap is wide and essentially matches
  OPD when the teacher is already reliable on the student histories (search/QA).
\end{itemize}

%% file: sections/related.tex
\section{Related Work}

\paragraph{Knowledge distillation for language models.}
Knowledge distillation transfers a teacher's behavior into a smaller
student~\citep{hinton2015distilling}, and sequence-level distillation extends this idea to autoregressive generation by training on teacher trajectories~\citep{kim2016sequence}. It underlies many strong small instructions and
reasoning models trained from teacher-written solutions or
rationales~\citep{mitra2024orca,yue2023mammoth,yu2023metamath,hsieh2023distilling}.
The standard recipe is efficient but off-policy: the student learns from the teacher
prefixes, then must recover from its own prefixes at inference. On-policy distillation addresses this by querying the teacher on student-sampled prefixes:
MiniLLM uses reverse KL~\citep{gu2024minillm}, GKD mixes student and teacher
generations~\citep{agarwal2024gkd}, DistiLLM refines the divergence and replay
design~\citep{ko2024distillm,ko2025distillm2}, speculative distillation interleaves
teacher and student tokens~\citep{xu2024speculativekd}, and recent practice frames
OPD as dense supervision compared with sparse RL~\citep{thinkingmachines2025onpolicy}.
Recent extensions study black-box OPD and on-policy context
distillation~\citep{ye2025blackboxopd,ye2026opcd}.
By analogy, a stronger teacher is not always a better teacher~\citep{xu2024stronger};
there the concern is teacher and data selection for single-turn instruction tuning,
whereas our reliability notion concerns \emph{where} the teacher's per-step conditional
is queried, so we take it as a suggestive parallel rather than direct support. We extend this line from single-turn generation to
multi-turn training from a fixed teacher-trajectory pool, where the key object is not only the teacher target but the prefix distribution on which it is
queried.

\paragraph{RL algorithms for LLM post-training.}
RL post-training optimizes scalar feedback rather than teacher conditionals. RLHF
fits preference rewards, often Bradley--Terry models~\citep{bradley1952rank}, and
optimizes them with PPO~\citep{schulman2017proximal}, as in instruction-following and
helpful-harmless assistants~\citep{ouyang2022training,bai2022training}. Direct
preference objectives simplify this pipeline by replacing online reward optimization
with contrastive or implicit-reward losses such as SLIC-HF~\citep{zhao2023slic},
DPO~\citep{rafailov2023direct}, IPO~\citep{azar2023general}, and
GPO~\citep{tang2024generalized}, with iterative and online variants studied under
KL-regularized preference learning~\citep{xiong2024iterative,dong2024rlhf}. For
reasoning, o1/DeepSeek-R1-style systems increasingly use verifiable or process
rewards~\citep{jaech2024openai,deepseekai2025deepseekr1incentivizingreasoningcapability,wang2023mathshepherd,zhang2024entropyregularizedprocessrewardmodel},
giving rise to scalable methods such as GRPO~\citep{shao2024deepseekmath} and
DAPO~\citep{yu2025dapo}. Critic-free variants revisit
REINFORCE~\citep{williams1992simple}, including ReMax~\citep{li2023remax},
RLOO-style baselines~\citep{ahmadian2024back,kool2019buy}, and
Reinforce-Rej~\citep{xiong2025minimalist}; RAFT and rejection sampling are the
binary-reward limit, where successful trajectories are selected and
imitated~\citep{dong2023raft,liu2023statistical,xiong2025minimalist}. ReOPD is
orthogonal to this optimizer family: it keeps dense teacher supervision and asks
which replayed prefixes make that supervision reliable.

\paragraph{Reasoning and agentic LLMs.}
RL-style post-training also extends to agents that interact with tools and
environments~\citep{yao2022react,schick2023toolformer,nakano2021webgpt}, including
search-augmented agents~\citep{jin2025search} and tool-integrated mathematical
reasoning~\citep{gou2023tora,mint2024a}. Other work improves reasoning through
self-hinting~\citep{liao2026selfhinting}, budget-aware elastic
reasoning~\citep{xu2025elasticreasoning}, online experiential
learning~\citep{ye2026onlineexperiential}, or distills teacher-generated
reasoning-and-acting traces into smaller agents~\citep{chen2023fireact,chen2024agentflan}.
These works motivate our setting but optimize different signals: scalar rewards,
advantages, filtered rollouts, or fixed teacher trajectories. We retain the dense
conditional target of distillation while avoiding live environment interaction, so the central design variable becomes which replayed prefixes should carry training
mass.

\paragraph{Self-training and rejection sampling.}
Self-training methods repeatedly turn model samples into new training data.
Expert-iteration methods imitate successful search or sampling
solutions~\citep{anthony2017thinking}, STaR bootstraps from self-generated
rationales~\citep{zelikman2022star}, and ReST-style methods alternate generation,
scalar-feedback filtering, and retraining~\citep{gulcehre2023reinforced,singh2023beyond}.
In LLM post-training, rejection sampling and RAFT instantiate the same principle by
imitating reward-ranked trajectories~\citep{dong2023raft,liu2023statistical,xiong2025minimalist},
with refinements for variance, curricula, and agentic
self-improvement~\citep{yao2025optimizing,zhao2024automatic,aksitov2023rest};
reflective methods revise trajectories from feedback~\citep{shinn2024reflexion}. ReOPD also reuses
offline traces, but the trace is not merely accepted or rejected: it becomes a site
where a teacher's condition is queried, so the effective prefix distribution is the
core variable.

\paragraph{Distribution shift and compounding errors.}
The prefix trap inherits the classic sequential-learning problem: behavior cloning on expert states suffers covariate shift once the learner deviates, which DAgger
addresses by training on the learner's induced state distribution~\citep{ross2011reduction}.
Autoregressive generation has the same exposure-bias structure; scheduled sampling,
sequence-level training, and Professor Forcing reduce the mismatch between
teacher-forced training and free-running inference~\citep{bengio2015scheduled,ranzato2016sequence,lamb2016professor}.
LLM alignment revisits this tension at scale: on-policy, even suboptimal, samples can
help preference fine-tuning~\citep{tajwar2024preference}, but on-policy data is not
uniformly best and depends on the alignment stage~\citep{sun2025onpolicydata}. Recent reasoning distillation similarly identifies a dual exposure bias: teacher-forced
traces mismatch student inference, while fully student-generated contexts can make
teacher intervention unreliable~\citep{wang2026backtracking}. Our formulation makes this duality explicit as a two-sided shift: student relevance pulls prefixes toward
the learner, while teacher reliability pulls them back toward the teacher-supported
region.

%% file: sections/formulation.tex
\section{Problem Formulation and Analysis}
\label{sec:formulation}

We formulate multi-turn on-policy distillation in an interactive environment.
At each OPD update, let \(\pi_{\theta_{\mathrm{old}}}\) denote the current
student policy used for data collection, and let \(\pi_\theta\) denote the
student policy being optimized. For each input \(x\sim\mathcal D\), the agent
interacts with an environment over \(T\) decision steps.

At interaction step \(t\), the agent observes an interaction history
\(H_t=(O_1,A_1,O_2,A_2,\ldots,A_{t-1},O_t)\), where \(O_t\) is the current
observation and \(A_s\) is the action taken at step \(s\). Given \(H_t\), a policy
\(\pi\) chooses an action \(A_t\sim\pi(\cdot\mid x,H_t)\), and the environment returns
the next observation \(O_{t+1}\sim\mathcal E(\cdot\mid x,H_t,A_t)\). Thus a step is an
interaction step consisting of a policy action followed by an environment response,
rather than merely a token-generation step.

For any policy \(\pi\), the policy-environment system induces a step-\(t\) history
occupancy distribution
\(d_{\pi,\mathcal E}^{t}(h_t\mid x)=\Pr_{\pi,\mathcal E}(H_t=h_t\mid x)\), which
factorizes recursively: if \(h_{t+1}=(h_t,a_t,o_{t+1})\), then
\[
d_{\pi,\mathcal E}^{t+1}(h_{t+1}\mid x)
=
d_{\pi,\mathcal E}^{t}(h_t\mid x)\,
\pi(a_t\mid x,h_t)\,
\mathcal E(o_{t+1}\mid x,h_t,a_t).
\]
In particular, the current student \(\pi_{\theta_{\mathrm{old}}}\) induces the
student interaction occupancy
\(d_{\theta_{\mathrm{old}}}^{t}\equiv d_{\pi_{\theta_{\mathrm{old}}},\mathcal E}^{t}\),
while the teacher \(\pi_T\) induces the teacher interaction occupancy
\(d_T^{t}\equiv d_{\pi_T,\mathcal E}^{t}\). Once a history \(h_t\) is collected, the
teacher is queried for a target action distribution \(\pi_T(\cdot\mid x,h_t)\).
Therefore, in interactive OPD the design choice is not only which target
distribution to distill from, but also which interaction histories to query the
teacher on.

\paragraph{Ideal interactive improvement objective.}

Let \(q_t^\star(\cdot\mid x,h_t)\) denote the ideal improvement target at interaction
step \(t\) and history \((x,h_t)\). The ideal objective evaluates the updated student
\(\pi_\theta\) on histories that the current student \(\pi_{\theta_{\mathrm{old}}}\)
will actually encounter under the environment interaction:
\begin{equation}
\label{eq:ideal}
\mathcal R^\star(\theta;\theta_{\mathrm{old}})
=
\mathbb E_{x\sim\mathcal D}
\left[
\sum_{t=1}^{T}
\alpha_t\,
\mathbb E_{H_t\sim d_{\theta_{\mathrm{old}}}^{t}(\cdot\mid x)}
\left[
\ell\left(
\pi_\theta(\cdot\mid x,H_t),
q_t^\star(\cdot\mid x,H_t)
\right)
\right]
\right],
\end{equation}
where \(\alpha_t\ge 0\) controls the relative importance of different
interaction steps; unless a schedule is specified we take \(\alpha_t\equiv 1\)
(uniform across steps), so that ReOPD's reliability signal -- including its step-decay --
is carried entirely by the weight \(w_t\) (Section~\ref{sec:aopd}), not by
\(\alpha_t\). This objective captures student relevance: the student should
improve on histories generated by its own interaction with the environment.

In practice, \(q_t^\star\) is unavailable. OPD therefore replaces it with the
teacher distribution \(\pi_T(\cdot\mid x,h_t)\). However, the teacher is now
queried on histories generated by a particular roll-in process. If those
histories are far from the teacher's reliable interaction support, the teacher's
action distribution may not be a trustworthy improvement signal.

\paragraph{From online interaction to offline reuse.}

The occupancies \(d_{\pi,\mathcal E}^{t}\) and the ideal objective
\(\mathcal R^\star\) are defined through environment interaction, but they serve
only as \emph{conceptual targets}. Realizing them online would require, at every
update, rolling the current student through the environment and querying the
teacher afresh at each visited history -- precisely the per-step environment and
teacher cost we wish to avoid. We instead train off-environment from a fixed pool
of pre-collected teacher trajectories: at each recorded history \(h_t\) the
student proposes an action and the teacher conditional
\(\pi_T(\cdot\mid x,h_t)\) supplies the target, while no environment query is
needed because the prefix and observations are replayed from the trace.
Consequently the roll-in is the teacher pool, so the collected histories
follow \(\mathcal P_t\approx d_T^{t}\) rather than the student occupancy
\(d_{\theta_{\mathrm{old}}}^{t}\). The student occupancy that the ideal objective
targets is therefore \emph{never sampled}; obtaining it is exactly what
online interaction would buy. This is the crux of the offline setting: we must
approximate an objective defined over \(d_{\theta_{\mathrm{old}}}^{t}\) using
only samples from \(d_T^{t}\), which is what the reweighting below accomplishes.

\paragraph{Collected and effective history distributions.}

Write \(\mathcal P_t(\cdot\mid x)\) for the roll-in distribution histories are collected
from; in the offline setting it is the fixed teacher pool, so
\(\mathcal P_t\approx d_T^{t}\). A general weighted OPD objective is
\begin{equation}
\label{eq:weighted}
\mathcal L_{\mathcal P,w}(\theta;\theta_{\mathrm{old}})
=
\mathbb E_{x\sim\mathcal D}
\left[
\sum_{t=1}^{T}
\alpha_t\,
\mathbb E_{H_t\sim \mathcal P_t(\cdot\mid x)}
\left[
w_t(x,H_t)\,
\ell\left(
\pi_\theta(\cdot\mid x,H_t),
\pi_T(\cdot\mid x,H_t)
\right)
\right]
\right],
\end{equation}
where the weights are normalized per \((x,t)\), i.e.\
\(\mathbb E_{H_t\sim \mathcal P_t(\cdot\mid x)}[w_t(x,H_t)]=1\), so that \(w_t\) reshapes
\(\mathcal P_t\) into the per-step distribution \(\rho_t\); the position-only schedule of
Section~\ref{sec:aopd}, constant across histories at each \(t\), is instead normalized
across positions, as detailed there. Then \(w_t\) induces
an effective interaction-prefix distribution
\(\rho_t(h_t\mid x)=w_t(x,h_t)\,\mathcal P_t(h_t\mid x)\), and the objective becomes
\begin{equation}
\label{eq:rho-obj}
\mathcal L_{\rho}(\theta;\theta_{\mathrm{old}})
=
\mathbb E_{x\sim\mathcal D}
\left[
\sum_{t=1}^{T}
\alpha_t\,
\mathbb E_{H_t\sim \rho_t(\cdot\mid x)}
\left[
\ell\left(
\pi_\theta(\cdot\mid x,H_t),
\pi_T(\cdot\mid x,H_t)
\right)
\right]
\right].
\end{equation}
Thus the central design problem is not the raw weight \(w_t\), but the effective
history distribution \(\rho_t\).

\paragraph{Two sources of mismatch.}

The practical objective differs from the ideal interactive improvement objective
for two reasons. First, the effective history distribution \(\rho_t\) may differ
from the student-environment occupancy \(d_{\theta_{\mathrm{old}}}^{t}\).
Second, even when histories are relevant to the student, the teacher may be
unreliable on histories far from its own interaction support.

Define the step-wise teacher reliability error
\[
\epsilon_{T,t}^{\theta}(x,h_t)
=
\left|
\ell\left(
\pi_\theta(\cdot\mid x,h_t),
q_t^\star(\cdot\mid x,h_t)
\right)
-
\ell\left(
\pi_\theta(\cdot\mid x,h_t),
\pi_T(\cdot\mid x,h_t)
\right)
\right|.
\]
\begin{assumption}[Bounded improvement loss]
\label{asm:bounded}
The loss against the ideal target is uniformly bounded: there exists a constant
\(B<\infty\) such that
\(\bigl|\ell\left(\pi_\theta(\cdot\mid x,h_t),q_t^\star(\cdot\mid x,h_t)\right)\bigr|\le B\)
for all \((x,h_t,t)\).
\end{assumption}

Assumption~\ref{asm:bounded} is mild and covers the losses we use. It holds
automatically for any bounded divergence -- e.g.\ total variation (\(B=1\)) or
Jensen--Shannon (\(B=\log 2\)). For the per-token KL used in
Section~\ref{sec:aopd}, it holds whenever the target conditional is bounded below on
its support, \(q_t^\star(a\mid x,h_t)\ge p_{\min}>0\), since then
\(D_{\mathrm{KL}}(\pi_\theta\,\|\,q_t^\star)\le\log(1/p_{\min})=:B\); a softmax with
bounded logits (equivalently a temperature or label-smoothing floor) guarantees such
a \(p_{\min}\). Only the ideal-target loss needs to be bounded: the teacher-target
loss enters the bound solely through the difference \(\epsilon_{T,t}^{\theta}\).

\begin{proposition}[Two-sided decomposition]
\label{prop:bound}
Recall the ideal objective
\(\mathcal R^\star(\theta;\theta_{\mathrm{old}})\)~(\plaineqref{eq:ideal}), which scores
\(\pi_\theta\) against the ideal target \(q_t^\star\) on the student occupancy
\(d_{\theta_{\mathrm{old}}}^{t}\), and the replayed objective
\(\mathcal L_\rho(\theta;\theta_{\mathrm{old}})\)~(\plaineqref{eq:rho-obj}), which scores
\(\pi_\theta\) against the teacher \(\pi_T\) on the effective prefix distribution
\(\rho_t=w_t\,\mathcal P_t\). Under Assumption~\ref{asm:bounded} (with bound \(B\)),
\begin{equation}
\label{eq:bound}
\left|
\mathcal R^\star(\theta;\theta_{\mathrm{old}})
-
\mathcal L_\rho(\theta;\theta_{\mathrm{old}})
\right|
\le
\mathbb E_{x\sim\mathcal D}
\left[
\sum_{t=1}^{T}
\alpha_t
\left\{
2B\,
\mathrm{TV}
\left(
d_{\theta_{\mathrm{old}}}^{t}(\cdot\mid x),
\rho_t(\cdot\mid x)
\right)
+
\mathbb E_{H_t\sim \rho_t(\cdot\mid x)}
\left[
\epsilon_{T,t}^{\theta}(x,H_t)
\right]
\right\}
\right].
\end{equation}
Here \(\mathrm{TV}\) is the total-variation distance, \(\alpha_t\ge0\) the step
coefficients of~(\plaineqref{eq:ideal}), and
\(\epsilon_{T,t}^{\theta}(x,h_t)=\bigl|\ell(\pi_\theta,q_t^\star)-\ell(\pi_\theta,\pi_T)\bigr|\)
the step-wise teacher-reliability error, the gap between the ideal- and teacher-target
losses at \((x,h_t)\).
\end{proposition}

\begin{proof}
Fix \(x\) and \(t\), and abbreviate the ideal-target and teacher-target losses by
\(f_t(h)=\ell(\pi_\theta(\cdot\mid x,h),q_t^\star(\cdot\mid x,h))\) and
\(g_t(h)=\ell(\pi_\theta(\cdot\mid x,h),\pi_T(\cdot\mid x,h))\), so that
\(\mathcal R^\star=\mathbb E_{x}\bigl[\sum_t\alpha_t\,\mathbb E_{d_{\theta_{\mathrm{old}}}^{t}}[f_t]\bigr]\)
and \(\mathcal L_\rho=\mathbb E_{x}\bigl[\sum_t\alpha_t\,\mathbb E_{\rho_t}[g_t]\bigr]\).
Adding and subtracting \(\mathbb E_{\rho_t}[f_t]\),
\[
\mathbb E_{d_{\theta_{\mathrm{old}}}^{t}}[f_t]-\mathbb E_{\rho_t}[g_t]
=
\underbrace{\bigl(\mathbb E_{d_{\theta_{\mathrm{old}}}^{t}}[f_t]-\mathbb E_{\rho_t}[f_t]\bigr)}_{\text{occupancy shift}}
+
\underbrace{\bigl(\mathbb E_{\rho_t}[f_t]-\mathbb E_{\rho_t}[g_t]\bigr)}_{\text{teacher reliability}}.
\]
For the occupancy-shift term, Assumption~\ref{asm:bounded} gives
\(\|f_t\|_\infty\le B\), and since
\(\sum_h\bigl|d_{\theta_{\mathrm{old}}}^{t}(h\mid x)-\rho_t(h\mid x)\bigr|=2\,\mathrm{TV}(d_{\theta_{\mathrm{old}}}^{t},\rho_t)\),
\[
\bigl|\mathbb E_{d_{\theta_{\mathrm{old}}}^{t}}[f_t]-\mathbb E_{\rho_t}[f_t]\bigr|
\le
\|f_t\|_\infty\sum_h\bigl|d_{\theta_{\mathrm{old}}}^{t}(h\mid x)-\rho_t(h\mid x)\bigr|
\le
2B\,\mathrm{TV}\bigl(d_{\theta_{\mathrm{old}}}^{t}(\cdot\mid x),\rho_t(\cdot\mid x)\bigr).
\]
For the teacher-reliability term, Jensen's inequality and the definition of
\(\epsilon_{T,t}^{\theta}\) give
\(\bigl|\mathbb E_{\rho_t}[f_t-g_t]\bigr|\le\mathbb E_{\rho_t}[|f_t-g_t|]=\mathbb E_{\rho_t}[\epsilon_{T,t}^{\theta}]\).
Applying the triangle inequality, multiplying by \(\alpha_t\ge0\), summing over \(t\),
and taking \(\mathbb E_{x\sim\mathcal D}\) establishes~(\plaineqref{eq:bound}).
\end{proof}
The first term is the student interaction-occupancy mismatch: it is small when the
effective training histories match the histories that the current student will
encounter through environment interaction. The second term is the teacher
reliability mismatch: it is small when the teacher is queried on histories where its
action distribution is a reliable proxy for the ideal improvement target.

Fully student-on-policy OPD (\(\rho_t=d_{\theta_{\mathrm{old}}}^{t}\) for all \(t\))
zeroes the occupancy term but not the reliability one: student-generated actions can
steer the environment into histories unlikely under the teacher's own interaction
process, especially at later steps where its conditional is a poorly calibrated
target. Teacher-forced histories (\(\rho_t=d_T^{t}\)) do the reverse -- reliable near
the teacher's own occupancy, but not where the student actually goes at test time.
Neither extreme is uniformly optimal, so the effective distribution should balance
student relevance against teacher reliability.

\paragraph{Reliability-aware interaction-prefix distribution.}

The preceding decomposition motivates choosing, for each interaction step, an
effective history distribution
\[
\rho_t^\star
\in
\arg\min_{\rho_t}
\left\{
\lambda_{\mathrm{stu},t}\,
D\left(
\rho_t(\cdot\mid x),
d_{\theta_{\mathrm{old}}}^{t}(\cdot\mid x)
\right)
+
\lambda_{\mathrm{tea},t}\,
\mathcal E_T(\rho_t;x)
\right\},
\]
where \(D\) measures mismatch to the current student's occupancy and
\(\mathcal E_T(\rho_t;x)\) measures expected teacher unreliability under \(\rho_t\).
Since the true reliability error is generally unobserved, we use teacher support as a
practical surrogate, giving the bridge objective
\begin{equation}
\label{eq:bridge-obj}
\rho_t^\star
\in
\arg\min_{\rho_t}
\left\{
\lambda_{\mathrm{stu},t}\,
D_{\mathrm{KL}}
\left(
\rho_t(\cdot\mid x)
\,\Vert\,
d_{\theta_{\mathrm{old}}}^{t}(\cdot\mid x)
\right)
+
\lambda_{\mathrm{tea},t}\,
D_{\mathrm{KL}}
\left(
\rho_t(\cdot\mid x)
\,\Vert\,
d_T^{t}(\cdot\mid x)
\right)
\right\}.
\end{equation}
When the supports overlap, the solution is the geometric bridge
\begin{equation}
\label{eq:bridge}
\rho_t^\star(h_t\mid x)
\propto
\left[
d_{\theta_{\mathrm{old}}}^{t}(h_t\mid x)
\right]^{\gamma_t}
\left[
d_T^{t}(h_t\mid x)
\right]^{1-\gamma_t},
\qquad
\gamma_t
=
\frac{\lambda_{\mathrm{stu},t}}{\lambda_{\mathrm{stu},t}+\lambda_{\mathrm{tea},t}}.
\end{equation}
When \(\gamma_t=1\) the bridge becomes fully student-on-policy; when \(\gamma_t=0\) it
becomes teacher-supported roll-in. Intermediate values select histories that are both
likely under the current student's environment interaction and not too far from the
teacher's reliable interaction support.

\paragraph{Two regimes.}

The bound~(\plaineqref{eq:bound}) sums two competing terms: the occupancy term is
minimized by \(\gamma_t\!\to\!1\) (fully student-on-policy) and the reliability term by
\(\gamma_t\!\to\!0\) (teacher-supported roll-in), so which one dominates sets the right
\(\gamma_t\). Crucially the balance is \emph{per step}: the teacher-reliability term
grows with depth, as roll-ins drift into histories the teacher rarely visits, so the
ideal \(\gamma_t\) \emph{decreases along the trajectory} -- student-relevant
(\(\gamma_t\!\to\!1\)) at early steps, teacher-supported (\(\gamma_t\!\to\!0\)) only at
the deep steps where the teacher signal has degraded. This yields two regimes. When the
teacher stays reliable on student-induced histories -- a capable teacher close to the
student, short horizons, or early steps -- the occupancy term dominates, \(\gamma_t\!\to\!1\)
is best, and ordinary student-on-policy OPD is already strong. When the teacher's signal
degrades off its own support -- a large teacher--student gap, long horizons, or later
steps -- lowering \(\gamma_t\) \emph{at those steps} reduces the gap even though the
resulting histories are \emph{less} student-on-policy, which is counterintuitive if one
reads OPD as merely ``make it on-policy''. Section~\ref{sec:aopd} approximates this
depth-dependent retreat without tuning \(\gamma_t\) directly -- shifting supervision
mass toward earlier steps rather than reshaping the roll-in within each step --
applying on the fixed teacher-pool roll-in a step-decay whose steepness grows with the
teacher--student gap.

\paragraph{From distribution to weight.}

When histories come from \(\mathcal P_t\), any effective distribution
\(\rho_t^\star\) is realized by the normalized importance weight
\(w_t^\star=\rho_t^\star/\mathcal P_t\) (with
\(\mathbb E_{H_t\sim\mathcal P_t}[w_t^\star]=1\)), which recovers the weighted
objective~(\plaineqref{eq:weighted}) and can be applied equivalently by reweighting or
by sampling prefixes in proportion to it. The weight is thus an implementation device,
not the conceptual core: the principle is to choose, at each step, an effective history
distribution that balances student relevance and teacher reliability. Because exact
density ratios are high-variance, the concrete step-decay schedule we use in
practice is developed in Section~\ref{sec:aopd}. After optimizing the weighted
objective the student is refreshed,
\(\theta_{\mathrm{new}}\leftarrow\arg\min_{\theta}\mathcal L_{\rho}(\theta;\theta_{\mathrm{old}})\),
and the next OPD iteration sets
\(\theta_{\mathrm{old}}\leftarrow\theta_{\mathrm{new}}\).

\paragraph{Interpretation.}
On-policy distillation is not automatically optimal merely because histories are
sampled from the student: the history distribution sets both where the student is
trained and where the teacher is asked to supervise, so when the teacher is unreliable
on student-induced histories a reliability-aware distribution close to both occupancies
is the safer target.

%% file: sections/aopd.tex
\section{Replayed-Prefix On-Policy Distillation}
\label{sec:aopd}

We now turn the principle of the previous section into a concrete algorithm,
\emph{Replayed-Prefix On-Policy Distillation} (ReOPD), realized off-environment from a fixed
pool of teacher trajectories. The analysis reduces
multi-turn OPD to a single design choice: at each step, pick an effective prefix
distribution \(\rho_t\) that balances student relevance against teacher
reliability. ReOPD realizes this with two ingredients:

\begin{enumerate}[leftmargin=1.6em,itemsep=2pt,topsep=2pt]
  \item \textbf{Teacher-forced prefix, student action, teacher supervision.} For a
  supervised step \(t\), the prefix \(h_t=(O_1,A_1,\dots,O_t)\) is taken
  \emph{entirely} from a pre-collected teacher trajectory -- every earlier action
  and every observation is the teacher's recorded one. At this prefix the
  \emph{student} produces its own action, and the teacher's recorded conditional
  \(\pi_T(\cdot\mid x,h_t)\) supplies the distillation target. Sweeping \(t\) over
  the trajectory covers all positions. Because observations are always replayed,
  no environment is ever queried.
  \item \textbf{Reliability-aware step schedule.} The student is on-policy only at the
  evaluated step, on top of a teacher-forced prefix. As the position moves deeper
  into the trajectory, the replayed teacher prefix becomes a higher-shift surrogate
  for the histories the student would have reached. A per-step schedule therefore
  emphasizes early, low-shift positions and de-emphasizes late, high-shift positions
  while keeping the teacher conditional as the distillation target.
\end{enumerate}

The key point is that on-policy supervision and off-environment training coexist:
the student's step-\(t\) action makes the target relevant to the student, while the
fully teacher-forced prefix keeps the roll-in within the teacher's occupancy
(\(\mathcal P_t\approx d_T^{t}\)) and removes any environment interaction. The step schedule
then controls how these teacher-supported histories stand in for the student occupancy
\(d_{\theta_{\mathrm{old}}}^{t}\).

\paragraph{Inputs and notation.}
We are given a teacher \(\pi_T\) and a pool \(\mathcal D_T=\{(x,h)\}\) of
trajectories generated by the teacher interacting with the environment, where each
\(h\) records the full interaction history including the environment observations.
In practice this pool requires no dedicated collection: when the teacher is trained
with on-policy RL, its training rollouts are exactly such trajectories, so
\(\mathcal D_T\) is a free by-product reused for distillation.

\paragraph{Off-environment prefix construction.}
For each supervised step \(t\), the prefix \(h_t=(O_1,A_1,\dots,O_t)\) is replayed
\emph{verbatim} from the teacher trace -- all earlier actions \(A_{<t}\) and all
observations \(O_{\le t}\) are the teacher's. The student only acts at step \(t\),
where it generates its own action
\(A_t=(a_t^1,\dots,a_t^{n_t})\sim\pi_{\theta_{\mathrm{old}}}(\cdot\mid x,h_t)\)
autoregressively and is supervised, \emph{per generated token}, by the teacher's
conditional along that action, \(\pi_T(\cdot\mid x,h_t,a_t^{<j})\). Because the
supervised token contexts \(a_t^{<j}\) are the student's own samples, the evaluated
step is genuinely on-policy even though the prefix is teacher-forced, and it enters the
loss below through these on-policy contexts. No action is ever executed against the
environment and no observation is ever generated, so no environment is queried. The
roll-in is thus the teacher occupancy, \(\mathcal P_t\approx d_T^{t}\); the two-sided
shift arises because these replayed prefixes are only a surrogate for the histories the
current student would reach, a mismatch that grows with depth and that the weight below
corrects by emphasizing the shared, low-shift portion of the trajectory.

\paragraph{From the bridge to a step-decaying weight.}
The bridge weight reshapes the collected roll-in \(\mathcal P_t\approx d_T^{t}\)
toward \(\rho_t^\star\). Substituting the bridge~(\plaineqref{eq:bridge}) gives a weight that
is a power of the student-to-teacher occupancy ratio,
\[
w_t(x,h_t)
\;\propto\;
\frac{\rho_t^\star(h_t\mid x)}{\mathcal P_t(h_t\mid x)}
\;=\;
\left(\frac{d_{\theta_{\mathrm{old}}}^{t}(h_t\mid x)}{d_T^{t}(h_t\mid x)}\right)^{\gamma_t}.
\]
Because the prefix \(h_t\) is teacher-recorded and both policies act on the
\emph{same} prefix, the environment factors cancel and only the model-induced part
remains, an accumulated student-to-teacher likelihood ratio over the prefix,
\[
\widehat r_t(x,h_t)
=
\prod_{s<t}
\frac{\pi_{\theta_{\mathrm{old}}}(a_s\mid x,h_s)}{\pi_T(a_s\mid x,h_s)}.
\]
This ratio is a genuine per-history quantity -- it varies across the histories at a
fixed step \(t\) -- so the weight that realizes the bridge inside the weighted
objective~(\plaineqref{eq:weighted}) is its per-\((x,t)\)-normalized power
\begin{equation}
\label{eq:exact-weight}
w_t(x,h_t)
=
\frac{\widehat r_t(x,h_t)^{\,\gamma_t}}
{\mathbb E_{H_t\sim\mathcal P_t(\cdot\mid x)}\!\left[\widehat r_t(x,H_t)^{\,\gamma_t}\right]},
\qquad
\mathbb E_{H_t\sim\mathcal P_t}[w_t]=1,
\end{equation}
which reshapes the pooled roll-in into the effective per-step distribution
\(\rho_t=w_t\,\mathcal P_t\) exactly as the general framework prescribes, upweighting
the histories the current student is relatively more likely to have produced. Because
\(\widehat r_t\) varies within a step this normalization is non-vacuous, so
Eq.~(\plaineqref{eq:exact-weight}) is ReOPD's exact, density-ratio realization of the
bridge; it is computable from the student and teacher log-probabilities on the
recorded prefix, at the price of a high-variance per-history ratio.

Each factor compares the student's and the teacher's probability of the teacher's
own recorded action \(a_s\): since the teacher selected \(a_s\), the ratio
\(\pi_{\theta_{\mathrm{old}}}(a_s)/\pi_T(a_s)\) is typically below one, so the product
\emph{shrinks} as the prefix lengthens. Equivalently, \(\log\widehat r_t\) is a
cumulative sum that becomes more negative with depth: the further the supervised
position sits in the trajectory, the less likely the current student would have
followed the same teacher prefix, and the larger the surrogate mismatch between the
pooled prefix and the student-relevant history distribution. In other words, the
reliability weight is, to first order, a monotone \emph{step-decaying} function of the
position index.

\begin{figure}[t]
    \centering
    \includegraphics[width=0.6\linewidth]{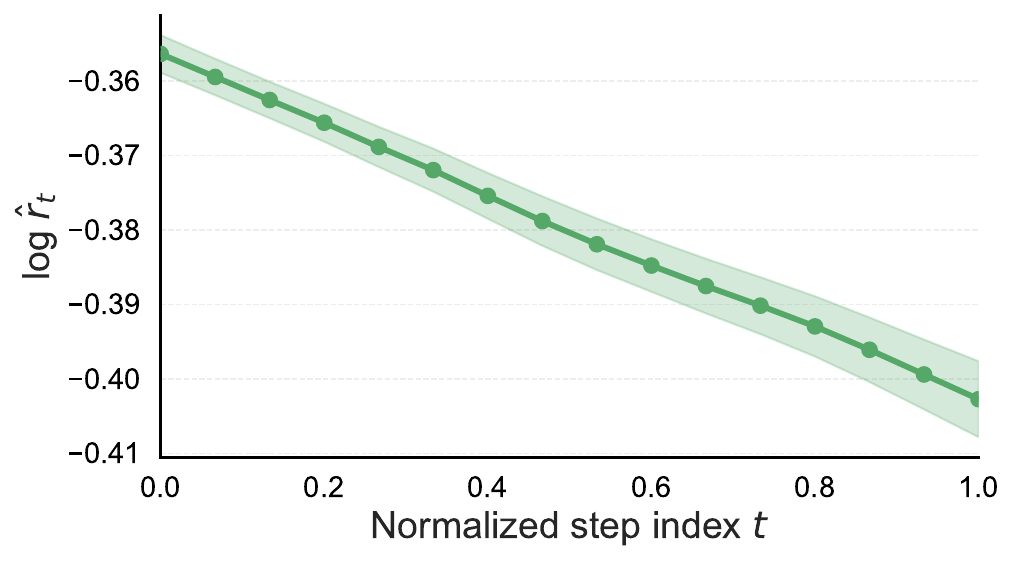}
    \vspace{-2mm}
  \caption[]{\textbf{Step index is a strong proxy for the likelihood-ratio weight.}
  For each supervised position we compute the per-position weight
  \(\widehat r_t=\prod_{s<t}\pi_{\theta_{\mathrm{old}}}(a_s)/\pi_T(a_s)\) -- the
  student-to-teacher likelihood ratio along the teacher prefix -- and plot it against
  the step index \(t\). The empirical weight decays with depth, because each
  factor compares the student's and teacher's probability of the teacher's own
  recorded action and is typically below one, so the product shrinks as the prefix
  lengthens. This makes step index a cheap proxy for the expensive prefix-level ratio
  and motivates replacing the explicit ratio with the one-parameter
  step-decay schedule \(\omega(t)=\kappa^{t}\) introduced below.}
  \label{fig:weight_decay}
  \vspace{-3mm}
\end{figure}

This holds empirically, not just to first order: measured directly along teacher
prefixes, \(\widehat r_t\) is tightly organized by step index
(Figure~\ref{fig:weight_decay}), so depth explains most of its variation and justifies
the design choice we make next -- replacing the expensive, high-variance per-prefix
ratio with a stable depth-based proxy.

Concretely, the surrogate keeps the monotone-decreasing \emph{shape} of the exact
weight~(\plaineqref{eq:exact-weight}) while dropping the noisy per-history factor
\(\widehat r_t\); both the shape and the right steepness follow from the exact weight
itself. Writing the average per-step teacher--student gap along the prefix as
\[
\bar c(x)
=
\mathbb E_{a_s\sim\pi_T}\!\left[
\log\frac{\pi_T(a_s\mid x,h_s)}{\pi_{\theta_{\mathrm{old}}}(a_s\mid x,h_s)}
\right]
\;\ge\;0
\qquad(\text{an average per-step } D_{\mathrm{KL}}(\pi_T\,\|\,\pi_{\theta_{\mathrm{old}}})),
\]
we have \(\log\widehat r_t\approx-(t-1)\,\bar c\), so the exact weight's depth profile
is geometric:
\begin{equation}
\label{eq:kappa-map}
\widehat r_t^{\,\gamma_t}\;\approx\;\kappa^{\,t}\ \text{(up to a constant)},
\qquad
\kappa=\exp\!\left(-\gamma_t\,\bar c\right)\in(0,1].
\end{equation}
This is the bridge-to-schedule map: the decay \emph{base} \(\kappa\) is not a free
prefactor but is pinned by the bridge exponent \(\gamma_t\) (how far \(\rho_t\) is
pushed toward the student occupancy) and the per-step gap \(\bar c\) (how fast the
teacher pool diverges from the student with depth), and shrinks -- steeper decay --
when either grows. Collapsing the profile to a single base \(\kappa\) treats the
per-step gap as approximately stationary across positions; when it instead grows with
depth, \(\kappa^{t}\) is a first-order (geometric-mean) approximation to the true
accumulated ratio, which Figure~\ref{fig:weight_decay} shows already tracks the measured
decay well. ReOPD therefore uses the one-parameter schedule
\[
\omega(t;\kappa)=\kappa^{\,t}
\qquad(\text{or any monotone decreasing schedule}),
\]
with \(\kappa\in(0,1]\) the single steepness knob; the overall scale cancels under
normalization, so -- unlike a raw power law -- there is no separate prefactor, and in
particular no scale symbol to collide with the bridge exponent \(\gamma_t\). This is the surrogate's reliability weight,
\(w_t=\omega(t;\kappa)=\kappa^{\,t}\). Because it depends only on the position index it
is constant across the histories at a fixed \(t\), so it is normalized \emph{over the
horizon} -- across positions, in proportion to \(\kappa^{\,t}\) -- rather than
per-\((x,t)\); a per-\((x,t)\) normalization would send a step-constant weight to
\(w_t\!\equiv\!1\) and erase the decay. This makes the trade-off between the two forms
explicit: both feed the same weighted objective~(\plaineqref{eq:weighted}) through the
effective distribution \(\rho_t\propto w_t\,\mathcal P_t\), but the exact
weight~(\plaineqref{eq:exact-weight}) varies \emph{within} each step (per-\((x,t)\)
normalization -- a genuine density ratio that reshapes histories at a fixed \(t\)),
whereas the step-decay surrogate varies \emph{across} steps (horizon normalization --
constant within a step) and so reallocates mass toward the early, low-shift positions.
They differ only in the axis along which \(w_t\) carries the reliability signal. The surrogate trades exactness for a single interpretable
knob: \(\kappa\!=\!1\) recovers uniform weighting across positions (no decay), and
smaller \(\kappa\) concentrates supervision on the early, more reliable steps. Both
forms point the same way, since \(\widehat r_t\) decays with depth; the surrogate is
the default we sweep in Section~\ref{sec:experiments}, realized by sampling positions
in proportion to it (equivalently, as a loss weight), as detailed next. More generally the weight can be
made data-dependent -- e.g.\ derived from the student's own sampling probabilities at
the evaluated step.

\paragraph{Weighted distillation objective.}
At each supervised step the student samples its own action
\(A_t=(a_t^1,\dots,a_t^{n_t})\sim\pi_{\theta_{\mathrm{old}}}(\cdot\mid x,H_t)\) and is
distilled against the teacher \emph{along that action}: the per-step loss is the
token-level KL accumulated over the student-sampled action tokens,
\[
\ell\big(\pi_\theta,\pi_T;x,H_t,A_t\big)
=
\sum_{j=1}^{n_t}
D_{\mathrm{KL}}\!\Big(
\pi_\theta(\cdot\mid x,H_t,a_t^{<j})
\,\big\|\,
\pi_T(\cdot\mid x,H_t,a_t^{<j})
\Big),
\]
whose conditioning contexts \(a_t^{<j}\) are on-policy, so the loss depends on the
student's action. The student then minimizes the reweighted step-wise objective
\[
\mathcal L(\theta)
=
\mathbb E_{x\sim\mathcal D}
\left[
\sum_{t=1}^{T}\alpha_t\,
\mathbb E_{H_t\sim\mathcal P_t(\cdot\mid x)}
\;
\mathbb E_{A_t\sim\pi_{\theta_{\mathrm{old}}}(\cdot\mid x,H_t)}
\Big[
w_t\,
\ell\big(\pi_\theta,\pi_T;x,H_t,A_t\big)
\Big]
\right].
\]
Here \(\alpha_t\) is the uniform base of Section~\ref{sec:formulation} and the
reliability signal is carried entirely by \(w_t\): the exact weight sets \(w_t\) by
Eq.~(\plaineqref{eq:exact-weight}) (per-\((x,t)\)-normalized), while the step-decay
surrogate sets \(w_t=\omega(t;\kappa)=\kappa^{t}\) (horizon-normalized). The effective
per-position mass is then \(\alpha_t\,w_t\,\mathcal P_t\propto w_t\,\mathcal P_t\), which
the next paragraph draws from directly.
After optimization the student is refreshed, \(\theta_{\mathrm{old}}\leftarrow\theta\),
and the next iteration re-evaluates the student's steps against the same teacher
prefixes, so supervision tracks the moving student while staying anchored to the
teacher pool.

\paragraph{From the objective to sampling.}
The design target is a \emph{distribution}, not a weight: objective~(\plaineqref{eq:rho-obj})
is the plain (unweighted) loss in expectation under \(\rho_t^\star\). Any unbiased
Monte-Carlo estimate of that expectation is a valid implementation, and there are two
natural ones.
\emph{(i) Sampling (direct).} Estimate~(\plaineqref{eq:rho-obj}) by drawing histories
from \(\rho_t^\star\) and averaging the \emph{unweighted} loss. Since the pool only
provides samples from \(\mathcal P_t\), we draw positions with
\(p_t\propto w_t\,\mathcal P_t=\rho_t^\star\) -- i.e.\ resample the pooled positions in
proportion to the step-decay schedule -- and apply the loss with no reweighting. This
is the direct estimator of~(\plaineqref{eq:rho-obj}): the schedule appears only in
\emph{which} positions are trained, matching how the offline pool is streamed and
touching only a subset of positions per step.
\emph{(ii) Weighting (importance).} Keep the pooled samples from \(\mathcal P_t\) and
instead multiply the loss by \(w_t\). The importance identity
\(\mathbb E_{H_t\sim\mathcal P_t}[w_t\,\ell]=\mathbb E_{H_t\sim\rho_t^\star}[\ell]\)
makes this the importance-weighted estimator of the \emph{same} expectation, at a
compute cost proportional to enumerating every position.
Sampling is therefore not an approximation of weighting: both are unbiased estimators
of the single distributional objective~(\plaineqref{eq:rho-obj}), with sampling the more
direct one. We use \emph{sampling} in all experiments and treat weighting as an
equivalent alternative.

\paragraph{Setting the decay.}
The only free knob is the steepness \(\kappa\). The map~(\plaineqref{eq:kappa-map})
makes it measurable: \(\kappa\approx\exp(-\gamma_t\bar c)\) is read off from the
per-step teacher--student gap \(\bar c\), an average
\(D_{\mathrm{KL}}(\pi_T\|\pi_{\theta_{\mathrm{old}}})\) estimated directly on the pool.
The weight always decreases with depth, so early, low-shift steps receive the most
mass and \(\kappa\) only sets how fast it falls off: \(\kappa\!=\!1\) is uniform
(OPD-like) and smaller \(\kappa\) concentrates supervision on the early steps. The two
regimes read off through \(\bar c\) -- a small gap gives \(\kappa\!\approx\!1\) and
gentle decay, a large gap (e.g.\ the wide teacher--student math setting) gives small
\(\kappa\) and steep decay -- and steepness is set empirically in
Section~\ref{sec:experiments}, entering only through the sampling probability
\(p_t\propto w_t\) (equivalently, a weight \(w_t\)).

\begin{algorithm}[t]
\caption{Replayed-Prefix On-Policy Distillation (ReOPD)}
\label{alg:aopd}
\begin{algorithmic}[1]
\Require teacher pool \(\mathcal D_T\), teacher \(\pi_T\), student \(\pi_\theta\),
step-decay schedule \(\omega(t;\kappa)=\kappa^{t}\)
\State initialize \(\theta_{\mathrm{old}}\leftarrow\theta\)
\For{each OPD iteration}
  \State assemble candidate positions \(\{(x,h_t)\}\) from pooled trajectories \Comment{prefix \(h_t\): teacher actions \(A_{<t}\), teacher observations \(O_{\le t}\)}
  \State draw positions with \(p_t\propto\omega(t;\kappa)\) \Comment{early, low-shift steps sampled more often}
  \For{each sampled position \((x,h_t)\)}
    \State student generates its action \(A_t=(a_t^1,\dots,a_t^{n_t})\sim\pi_{\theta_{\mathrm{old}}}(\cdot\mid x,h_t)\) \Comment{token contexts on-policy; no environment call}
    \State accumulate \(\sum_{j=1}^{n_t} D_{\mathrm{KL}}\!\big(\pi_\theta(\cdot\mid x,h_t,a_t^{<j})\,\|\,\pi_T(\cdot\mid x,h_t,a_t^{<j})\big)\) \Comment{per-token KL along the student's action \(A_t\)}
  \EndFor
  \State update \(\theta\) on the accumulated loss; \(\theta_{\mathrm{old}}\leftarrow\theta\)
\EndFor
\State \Return \(\pi_\theta\)
\end{algorithmic}
\end{algorithm}

\paragraph{Why this is more than ``make it on-policy''.}
Standard OPD pushes prefixes toward the student by re-rolling through a live
environment; ReOPD instead replays a teacher-forced prefix and corrects with a
reliability-aware step-decay schedule, achieving student relevance off-environment. The
principle is reliability-aware prefix distribution design, which we implement by
sampling prefixes along that schedule.

\paragraph{Relation to RL and distillation.}
ReOPD sits between reinforcement learning and distillation, and the bridge parameter
\(\gamma_t\) is precisely the slider between them. Like on-policy RL, the student
acts at the supervised step and is trained on its own actions, so the signal is
relevant to what the student will do; like distillation, the supervision is the
teacher's full per-step conditional \(\pi_T(\cdot\mid x,h_t)\) rather than a scalar
return. This is the sense in which on-policy distillation provides a \emph{dense}
improvement signal where RL provides a sparse one~\citep{thinkingmachines2025onpolicy}:
RL conveys a few bits per episode, whereas matching a distribution supervises every
token. In our occupancy language, \(\gamma_t\!\to\!1\) places the effective prefix
distribution at the student occupancy \(d_{\theta_{\mathrm{old}}}^{t}\) -- the regime
on-policy methods target -- while \(\gamma_t\!\to\!0\) places it at the teacher
occupancy \(d_T^{t}\) -- ordinary teacher-forced distillation -- and intermediate
values interpolate.

This should not be mistaken for off-policy or reward-weighted RL. We do not estimate
returns or advantages, apply no importance-sampling correction to a value function,
and never optimize a scalar reward; the target is a fixed teacher distribution and
the weight \(w_t\) encodes the reliability and student-relevance of the replayed
prefix, not a policy-gradient ratio. Nor do we require environment interaction: the
``on-policy'' part is the
student's action at a single step on a replayed prefix, not a fresh rollout. ReOPD is
thus best read as distillation with an on-policy supervised step and a
reliability-aware prefix distribution -- inheriting the relevance of RL-style
on-policy data and the density of distillation, while avoiding both the sparsity of
reward and the cost of online rollouts.

%% file: sections/exp.tex
\section{Experiments}
\label{sec:experiments}

\begin{table}[t]
    \caption{\textbf{Training parameters for different methods.} GRPO is only used to train the teacher model. And the teacher's rollouts for GRPO are reused for ReOPD. Search requires a large batch size for GRPO to reduce noise.}
    \label{tab:training_hyperparams}
    \scriptsize
    \centering
    \begin{tabular}{lccccc}
        \toprule
        \textbf{Hyper-parameter} & \textbf{Cold Start} & \textbf{SFT} & \textbf{GRPO} & \textbf{OPD} & \textbf{ReOPD} \\
        \midrule
        batch size & 64 & 64 & 32$\times$8 (n=8) for Math & 256$\times$1 (n=1) & 256$\times$1 (n=1) \\
        & & &  128$\times$4 (n=4) for Search \\
        lr & 5e-5 & 5e-5 & 1e-6 & 1e-6 & 1e-6 \\
        lr scheduler & cosine &  cosine & constant & constant & constant   \\
        min lr & 1e-6 & 1e-6 &  - & - & - \\
        warmup ratio & 0.1 & 0.1 & - & - & - \\
        epoch / steps & 3 epochs & 5 epochs & 200 steps & 200 steps & 200 steps \\
        clip ($\epsilon_{low}$, $\epsilon_{high}$) & - & - & (0.2, 0.28) & - & - \\
        temperature & - & - & 1.0 & 1.0 & 1.0 \\
        max\_gen\_tokens & - & - & 8192 for Math & 8192 for Math & 8192 for Math \\
        & - & - & 4196 for Search & 4196 for Search & 4196 for Search \\
        & - & - & - & 8192 for multi-envs & 8192 for multi-envs \\
        max number of tool calls & - & - & 16 for Math & 16 for Math & 16 for Math \\
        & - & - & 4 for Search & 4 for Search & 4 for Search \\
        & - & - & - & 16 for multi-envs & 16 for multi-envs \\
        concurrency of tool calls & - & - & 32 & 32 & - \\
        GPU mem for retriver & - & - & 80GB & 80GB & - \\
        \bottomrule
    \end{tabular}
\end{table}

\begin{table}[t]
    \caption{\textbf{Evaluation hyper-parameters for two environments.}}
    \label{tab:eval_hyperparams}
    \scriptsize
    \centering
    \begin{tabular}{lcc}
        \toprule
        \textbf{Hyper-parameter} & \textbf{Math with Python} & \textbf{Search} \\
        \midrule
        temperature & 1.0 & 1.0 \\
        max\_gen\_tokens & 16384 & 8192 \\
        accuracy avg@n & $n=8$ (AIME24, AIME25, AMC23) or $n=4$ (the rest) & $n=1$ \\
        max number of tool calls  & 16 & 4 \\
        \bottomrule
    \end{tabular}
\end{table}

\begin{table}[t]
    \caption{\textbf{The statistics of evaluation sets.}}
    \label{tab:eval_statistics}
    \centering
    \small
    \begin{tabular}{l|ccccccc}
        \toprule
        \textbf{Math task} & AIME24 & AIME25 & AMC23 & Minerva & Olympiad & MATH500 & Sum \\
        \textbf{Amount} & 30 & 30 & 40 & 272 & 675 & 500 & 1547 \\
        % \bottomrule
    \end{tabular}
    \vspace{1mm}
    \begin{tabular}{l|cccccccc}
        \toprule
        \textbf{Search tasks} & NQ & TriviaQA & PopQA & HotpotQA & 2Wiki & Musique & Bamboogle & Sum \\
        \textbf{Amount} & 3610 & 11313 & 14267 & 7405 & 12576 & 2417 & 125 & 51713 \\
        \bottomrule
    \end{tabular}
\end{table}

\begin{table}[t]
    \caption{\textbf{Distillation on mathematical reasoning.}
    Each student is trained and evaluated in the math (Python-tool) environment on six
    competition benchmarks, under teachers of increasing scale. Across teacher scales,
    ReOPD improves over student-on-policy OPD on mathematical reasoning -- the
    teacher-reliability regime of Section~\ref{sec:formulation} -- and stays close
    to OPD when the gap is small. The better of OPD and ReOPD in each column is in
    bold; GRPO is the teacher's own RL result, shown for reference.}
    \label{tab:main_math_single}
    \footnotesize
    \centering
    \begin{tabular}{lccccccc}
        \toprule
            \textbf{Method} & \textbf{AIME24} & \textbf{AIME25} & \textbf{AMC23} & \textbf{Minerva} & \textbf{Olympiad} & \textbf{MATH500} & \textbf{\textit{Avg.}} \\
        \midrule
            \multicolumn{8}{l}{\textit{Teacher model: Qwen3-4B-Instruct-2507}} \\
            GRPO & 42.1 & 32.1 & 79.7 & 37.9 & 56.1 & 85.1 & 55.5 \\
        \cmidrule(lr){1-8}
            \multicolumn{8}{l}{\textit{Student model: Qwen3-4B-Instruct-2507}} \\
            Base & 6.3 & 6.7 & 36.6 & 31.0 & 29.5 & 58.2 & 28.0 \\
            Cold Start & 22.1 & 20.8 & 66.6 & 32.4 & 49.0 & 79.1 & 45.0 \\
            SFT & 21.7 & 21.3 & 63.1 & 39.8 & 49.9 & 80.9 & 46.1  \\
            OPD & 35.4 & 29.2 & 77.8 & \textbf{43.3} & 58.2 & 86.9 & 55.1 \\
            \rowcolor{myblue}
            ReOPD & \textbf{40.8} & \textbf{31.3} & \textbf{80.0} & 42.9 & \textbf{59.4} & \textbf{88.9} & \textbf{57.2}  \\
        \midrule
        \midrule
            \multicolumn{8}{l}{\textit{Teacher model: Qwen3-8B}} \\
            GRPO & 28.3 & 23.3 & 73.1 & 38.5 & 52.0 & 84.3 & 49.9 \\
        \cmidrule(lr){1-8}
            \multicolumn{8}{l}{\textit{Student model: Qwen3-4B-Instruct-2507}} \\
            SFT & 22.1 & 22.9 & 63.8 & 36.0 & 44.7 & 80.8 & 45.0 \\
            OPD & 28.3 & 26.3 & 70.3 & 39.8 & 55.3 & \textbf{85.8} & 51.0 \\
            \rowcolor{myblue}
            ReOPD & \textbf{36.7} & \textbf{29.2} & \textbf{74.4} & \textbf{41.1} & \textbf{56.4} & 84.7 & \textbf{53.7} \\
        \midrule
        \midrule
            \multicolumn{8}{l}{\textit{Teacher model: Qwen3-30B-A3B-Instruct-2507}} \\
            GRPO & 47.9 & 35.0 & 91.3 & 46.4 & 63.7 & 92.1 & 62.7 \\
        \cmidrule(lr){1-8} 
            \multicolumn{8}{l}{\textit{Student model: Qwen3-4B-Instruct-2507}} \\
            OPD & 28.3 & 25.8 & \textbf{75.3} & \textbf{39.6} & 53.6 & 83.8 & 51.1 \\
            \rowcolor{myblue}
            ReOPD & \textbf{32.5} & \textbf{26.7} & 74.4 & 39.4 & \textbf{55.8} & \textbf{86.5} & \textbf{52.5} \\
        \cmidrule(lr){1-8} 
            \multicolumn{8}{l}{\textit{Student model: Qwen3-8B}} \\
            Base & 21.3 & 11.3 & 43.1 & 37.2 & 34.4 & 64.8 & 35.3 \\
            Cold Start & 23.9 & 19.6 & 67.8 & 38.6 & 49.0 & 82.0 & 46.8 \\
            OPD & 40.8 & 29.2 & 77.8 & \textbf{45.1} & \textbf{56.9} & \textbf{89.0} & 56.5 \\
            \rowcolor{myblue}
            ReOPD & \textbf{41.7} & \textbf{30.4} & \textbf{80.6} & 44.1 & 55.7 & 88.4 & \textbf{56.8} \\
        \bottomrule
    \end{tabular}
\end{table}

\begin{table}[t]
    \caption{\textbf{Distillation on search.} Each student is
    trained and evaluated in the search environment on general-QA and multi-hop-QA
    benchmarks ($^\dagger$ in-domain, $^*$ out-of-domain). The teacher stays in the
    same Qwen3-4B family and remains reliable on student-induced histories, so the
    student-occupancy shift dominates and ReOPD essentially matches student-on-policy
    OPD. The better of OPD and ReOPD in each column is in bold; GRPO is the teacher's
    own RL result, shown for reference.}
    \label{tab:main_search_single}
    \footnotesize
    \centering
    \begin{tabular}{lcccccccc}
        \toprule
            \multirow{3}{*}{\textbf{Method}} & \multicolumn{3}{c}{\textbf{General QA}} & \multicolumn{4}{c}{\textbf{Multi-Hop QA}} & \multirow{2}{*}{\textbf{\textit{Avg.}}} \\
        \cmidrule(lr){2-4} \cmidrule(lr){5-8}
            & \textbf{NQ$^\dagger$} & \textbf{TriviaQA$^*$} & \textbf{PopQA$^*$} & \textbf{HotpotQA$^\dagger$} & \textbf{2wiki$^*$} & \textbf{Musique$^*$} & \textbf{Bamboogle$^*$} \\
        \midrule
            \multicolumn{9}{l}{\textit{Teacher model: Qwen3-4B-Instruct-2507}} \\
            GRPO & 36.1 & 59.2 & 40.2 & 37.2 & 31.8 & 13.8 & 39.2 & 40.4 \\
        \cmidrule(lr){1-9} 
            \multicolumn{9}{l}{\textit{Student model: Qwen3-4B-Instruct-2507}} \\
            Base & 28.4 & 55.5 & 30.6 & 30.8 & \textbf{33.6} & 8.4 & 35.8 & 35.6 \\
            Cold Start & 26.2 & 48.6 & 33.6 & 24.7 & 26.6 & 5.8 & 24.8 & 32.1 \\
            OPD & 36.3 & \textbf{59.7} & 40.3 & \textbf{37.3} & 31.7 & \textbf{14.5} & \textbf{44.0} & \textbf{40.6} \\
            \rowcolor{myblue}
            ReOPD & \textbf{36.4} & 59.6 & \textbf{40.5} & \textbf{37.3} & 31.5 & 14.0 & 37.6 & 40.5 \\
        \midrule
        \midrule
            \multicolumn{9}{l}{\textit{Teacher model: Qwen3-8B}} \\
            GRPO & 36.6 & 59.3 & 40.2 & 36.7 & 34.7 & 13.8 & 39.2 & 41.0  \\
        \cmidrule(lr){1-9} 
            \multicolumn{9}{l}{\textit{Student model: Qwen3-4B-Instruct-2507}} \\
            OPD & \textbf{35.5} & 57.3 & 39.0 & \textbf{35.0} & \textbf{31.2} & \textbf{12.8} & \textbf{37.6} & \textbf{39.1} \\
            \rowcolor{myblue}
            ReOPD & 34.9 & \textbf{57.5} & \textbf{39.1} & 34.6 & 31.0 & 11.9 & \textbf{37.6} & 39.0 \\
        \bottomrule
    \end{tabular}
\end{table}

\begin{table}[t]
    \caption{\textbf{Multi-environment distillation: one student for both domains.}
    A \emph{single} Qwen3-4B student is trained jointly across the math and search
    environments and evaluated on both -- unlike the single-environment results
    (Tables~\ref{tab:main_math_single} and~\ref{tab:main_search_single}), where a
    separate student is trained per domain. The teacher, by contrast, is \emph{not}
    shared across domains: each environment has its own RL-trained Qwen3-4B teacher, so
    the GRPO reference columns for math and search are two different models. Distilled
    from its environment's teacher, the single joint student stays on par with
    student-on-policy OPD in each domain, showing that ReOPD's reliability-aware schedule
    carries over to joint training. Benchmarks are rows and methods are columns
    ($^\dagger$ in-domain, $^*$ out-of-domain); the student base model is
    Qwen3-4B-Instruct-2507, and GRPO is the teacher's own RL result for
    reference. For each benchmark the better of OPD and ReOPD is in bold.}
    \label{tab:main_multi}
    \footnotesize
    \centering
    \begin{tabular}{l c c c >{\columncolor{myblue}}c}
        \toprule
            \multirow{3}{*}{\textbf{Benchmark}} & \textbf{Teacher (Qwen3-4B)} & \multicolumn{3}{c}{\textbf{Student (Qwen3-4B)}} \\
        \cmidrule(lr){2-2} \cmidrule(lr){3-5}
             & \textbf{GRPO} & \textbf{Cold Start} & \textbf{OPD} & \textbf{ReOPD} \\
        \midrule
            \multicolumn{1}{l}{\textit{Mathematical reasoning}} & & & & \\
            AIME24            & 42.1 & 18.8 & 34.2 & \textbf{36.3} \\
            AIME25            & 32.1 & 23.8 & 29.6 & \textbf{30.8} \\
            AMC23             & 79.7 & 67.8 & \textbf{79.7} & 75.5 \\
            Minerva           & 37.9 & 40.5 & \textbf{44.5} & \textbf{44.5} \\
            Olympiad          & 56.1 & 48.9 & 57.1 & \textbf{57.4} \\
            MATH500           & 85.1 & 82.4 & \textbf{87.3} & \textbf{87.3} \\
            \quad\textit{Avg.}& 55.5 & 47.0 & \textbf{55.4} & 55.3 \\
        \midrule
            \multicolumn{1}{l}{\textit{Search}} & & & & \\
            NQ$^\dagger$       & 36.1 & 25.7 & 37.3 & \textbf{37.7} \\
            TriviaQA$^*$       & 59.2 & 51.3 & \textbf{61.0} & 60.7 \\
            PopQA$^*$          & 40.2 & 33.2 & 40.5 & \textbf{40.6} \\
            HotpotQA$^\dagger$ & 37.2 & 25.1 & \textbf{38.0} & 37.6 \\
            2wiki$^*$          & 31.8 & 27.1 & 31.6 & \textbf{32.1} \\
            Musique$^*$        & 13.8 & 6.3 & \textbf{14.5} & 13.8 \\
            Bamboogle$^*$      & 39.2 & 29.6 & \textbf{44.0} & 37.6 \\
            \quad\textit{Avg.} & 40.4 & 32.7 & \textbf{41.0} & \textbf{41.0} \\
        \bottomrule
    \end{tabular}
\end{table}

\noindent\paragraph{Models.}
We use Qwen3-family models~\citep{yang2025qwen3} as both teachers and students.
To vary the teacher--student capability gap, we distill from teachers of three scales --
Qwen3-4B-Instruct-2507, Qwen3-8B, and Qwen3-30B-A3B-Instruct-2507 -- into a
Qwen3-4B-Instruct-2507 student; we additionally use Qwen3-8B as a student under the
Qwen3-30B-A3B-Instruct-2507 teacher.

All models go through a cold start stage with supervised finetuning (SFT) on trajectories collected by a stronger model, which aims to teach the models to better utilize the available tools and output a correct format. Without the cold start stage, we observe that the model frequently ignores the Python tool for math tasks. After SFT, the teacher model is then trained with GRPO on the training set. Such a design aims to imitate a practical setting: We first apply RL to obtain a stronger teacher model, and then we distill from the teacher model to the student model. In short, the teacher model will go through SFT (cold start) and RL, while the student model is initialized from the cold start version.

By default, the teacher prefix pool is the set of agentic trajectories produced by the
teacher while being trained with on-policy RL (GRPO). In this way, we can reuse the trajectories from the training of the teacher model without any extra collection cost (Section~\ref{sec:aopd}).

\noindent\paragraph{Datasets \& environments.}
We train and evaluate in two agentic tasks, both following the same three-stage protocol: an SFT cold start on 2K teacher trajectories (for both teacher and student), GRPO~\citep{shao2024deepseekmath} training of the teacher, and student distillation (OPD or ReOPD) on that same prompt set. \emph{(i) Mathematical reasoning (Python-tool environment).} Following ReTool~\citep{feng2025retool}, cold-start traces are taken from ReTool and the 6.4K prompts for GRPO and distillation are taken from the DAPO training set~\citep{yu2025dapo}. We evaluate on six competition benchmarks -- AIME24~\citep{aime24}, AIME25~\citep{aime25}, AMC23~\citep{li2024numinamath}, Minerva Math~\citep{lewkowycz2022minerva}, OlympiadBench~\citep{he2024olympiadbench}, and MATH500~\citep{hendrycks2021math500}. \emph{(ii) Search (retrieval environment).} Following Search-R1~\citep{jin2025search}, prompts come from the merged NQ~\citep{kwiatkowski2019nq} and HotpotQA~\citep{yang2018hotpotqa} training sets: we randomly sample 2K prompts for cold-start trajectory generation by Qwen3-30B-A3B-Instruct-2507, and another 6.5K prompts for teacher GRPO and student distillation. For retrieval, we use the 2018 Wikipedia dump~\citep{karpukhin2020dense} as the knowledge source and E5~\citep{wang2022text} as the retriever, with top-3 retrieved passages. We evaluate on seven QA benchmarks -- general QA (NQ, TriviaQA~\citep{joshi2017triviaqa}, PopQA~\citep{mallen2023popqa}) and multi-hop QA (HotpotQA, 2WikiMultiHopQA~\citep{ho2020twowiki}, MuSiQue~\citep{trivedi2022musique}, Bamboogle~\citep{press2023bamboogle}) -- with NQ and HotpotQA in-domain ($^\dagger$) and the remaining five out-of-domain ($^*$). The statistics for the evaluation sets is listed in Table~\ref{tab:eval_statistics}

In the multi-environment setting, a single student is trained jointly on both environments' training data for both SFT and OPD/ReOPD and evaluated on all benchmarks. The training data from both environments are concatenated and shuffled. We still use the teacher model trained on the single environment, which is a practical setting that distills knowledge from domain-specific teacher to a general student.

\noindent\paragraph{Baselines.}
We compare ReOPD against the following methods. \emph{Base}: the untrained model. \emph{Cold
Start}: the model is trained with SFT for the cold start stage. This is also the initialized model for the following methods and ReOPD. \emph{SFT}: the model is trained on the trajectories collected from the teacher model (same as ReOPD) with SFT. \emph{OPD}: fully student-on-policy distillation, where prefixes are re-rolled by the current student through the live environment and the teacher supplies the per-step target. \emph{GRPO}: the teacher's own RL result, shown for reference as a soft upper bound at the teacher's scale, not a competing student.

\noindent\paragraph{Implementation details.}
In Table \ref{tab:training_hyperparams} and \ref{tab:eval_hyperparams}, we list the main experimental settings for training and evaluation. We only train for 200 steps for RL, since the training converges very soon. ReOPD and OPD share the same teacher target, batch size, number of update steps and so on. The only difference is the prefix distribution. Unless noted, ReOPD uses the sampling implementation with a step-decay schedule (Section~\ref{sec:aopd}) with $\kappa=0.6$. Following \citet{feng2025retool} and \citet{jin2025search}, we report the macro and micro averages for the math and search tasks, respectively. All experiments are conducted on 8$\times$H100 with Slime~\citep{slime_github}. The training can be done within 3 hours for ReOPD.

\begin{figure}[t]
    \centering
    \begin{subfigure}{0.48\textwidth}
        \centering
        \includegraphics[width=\linewidth]{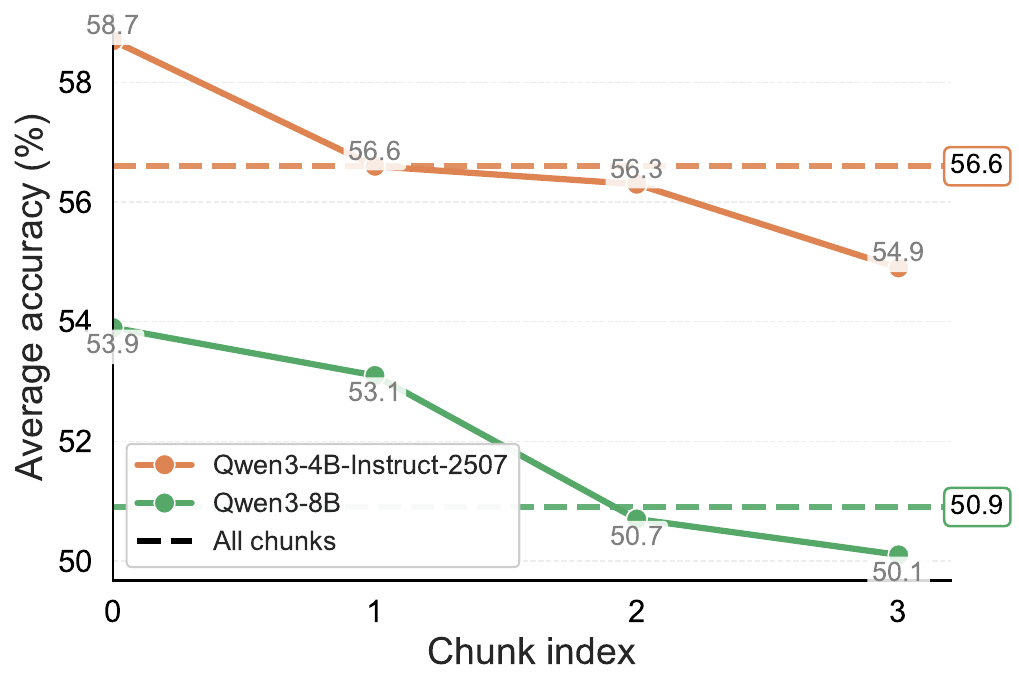}
        \caption{Sampling by chunk index.}
        \label{fig:chunk_index}
        %\vspace{-6mm}
    \end{subfigure}
    \begin{subfigure}{0.48\textwidth}
        \centering
        \includegraphics[width=\linewidth]{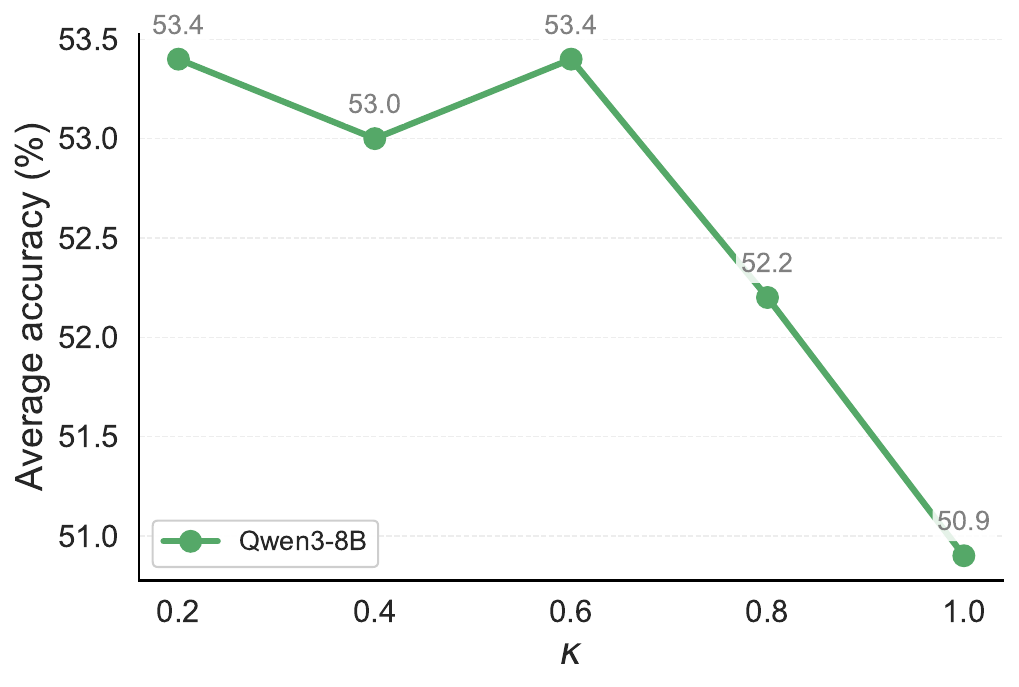}
        \caption{Sampling with decay $\kappa$.}
        \label{fig:sampling_decay}
        %\vspace{-6mm}
  \end{subfigure}
  \caption[]{\textbf{Favoring early, low-shift steps improves distillation.}
  Both panels use ReOPD's sampling implementation and put more training mass on early
  interaction steps, where the student stays close to the teacher and the two-sided
  shift is small. \textbf{(a)} \emph{Sampling by chunk}: drawing supervised positions
  from early chunks yields the best student, and pushing mass to later chunks degrades
  it. \textbf{(b)} \emph{Sampling with decay $\kappa$}: applying the step-decay
  \(\kappa^{t}\) as the sampling probability, a moderate decay outperforms uniform
  sampling (\(\kappa\!=\!1\)) and tracks the same trend. Final performance is governed
  by how strongly early steps are favored, validating the step-decay of
  Section~\ref{sec:aopd}. The student model is Qwen3-4B-Instruct-2507, and the models in the legends are teacher model.}
  \label{fig:poc_hints}
  \vspace{-4mm}
\end{figure}

\subsection{Main Results}

\paragraph{Single environment.}
Tables~\ref{tab:main_math_single} and~\ref{tab:main_search_single} compare ReOPD with fully online OPD in single-environment distillation. On mathematical reasoning, ReOPD consistently improves the student average over OPD across teacher settings. With a Qwen3-4B teacher and Qwen3-4B student, ReOPD improves the average from $55.1$ to $57.2$; with a Qwen3-8B teacher, it improves from $51.0$ to $53.7$; and with a Qwen3-30B-A3B teacher, it improves the Qwen3-4B student from $51.1$ to $52.5$. For the stronger Qwen3-8B student under the Qwen3-30B-A3B teacher, ReOPD also slightly improves the average from $56.5$ to $56.8$. These gains are largest when the teacher--student gap is wider, consistent with our teacher-reliability view: fully student-on-policy roll-ins can drift into histories where the teacher target is less reliable, while ReOPD stays anchored to teacher-supported prefixes.

\begin{wrapfigure}{r}{0.43\textwidth}
    \vspace{-6mm}
    \includegraphics[width=0.98\linewidth]{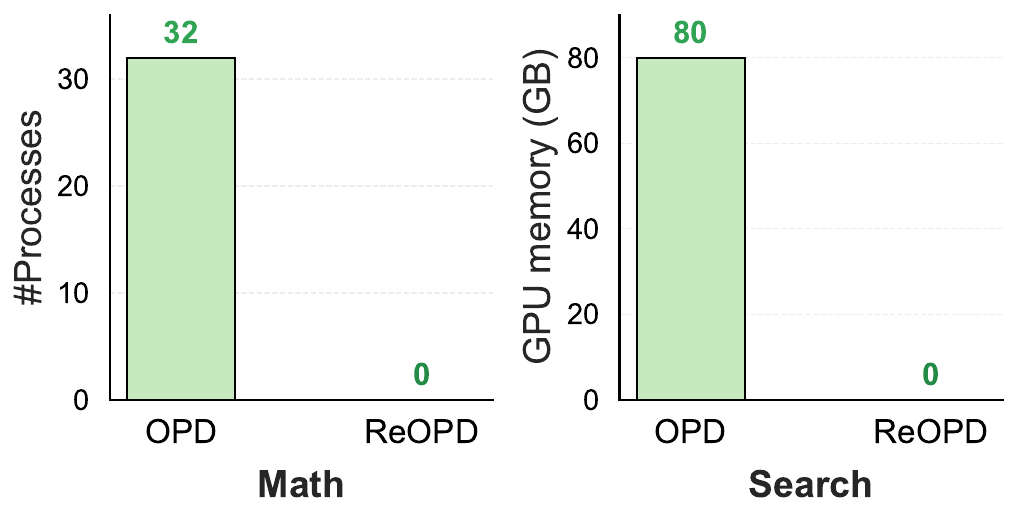} 
    \vspace{-3mm}
    \caption[]{\textbf{The resources used by different environments.} For math, the processes are utilized for the execution of Python code in parallel. For search, the GPU memory is utilized for deploying the embedding model and the storage of the Wikipedia embeddings.}
    \label{fig:process_memory}
    \vspace{-6mm}
\end{wrapfigure}

On search, ReOPD essentially matches OPD. With a Qwen3-4B teacher, OPD and ReOPD obtain $40.6$ and $40.5$ average accuracy, respectively; with a Qwen3-8B teacher, they obtain $39.1$ and $39.0$. This is the regime where the teacher remains reliable on student-induced histories, so the benefit of teacher-anchored replay is smaller. Together, the math and search results show that ReOPD preserves OPD-level accuracy in the reliable-teacher regime and improves over OPD when teacher reliability becomes the dominant constraint.

\paragraph{Multiple environments.}
Table~\ref{tab:main_multi} evaluates a single Qwen3-4B student trained jointly on the math and search environments. I.e. we distill the knowledge from two domain-specific teachers to a shared student, which is also known as multi-teacher on-policy distillation (MOPD). In this setting, the resources used by the environments grow quickly with an increasing number of environments. ReOPD remains on par with OPD in both domains while avoiding online environment interaction during student training. This indicates that the same replay-based training pipeline can combine heterogeneous environments without keeping all tools online during distillation. The result is especially useful operationally: each environment can contribute teacher rollouts to a shared offline pool, and the student can then be trained jointly from the merged data.

\paragraph{Efficiency.}
Figure~\ref{fig:overall} shows that ReOPD keeps the accuracy benefits of OPD while removing its main training-time bottleneck. Unlike OPD, which must execute fresh environment rollouts and tool calls during every student update, ReOPD replays teacher-recorded prefixes and therefore uses zero tool calls during student training. This also translates into faster updates: ReOPD is at least $4\times$ faster per rollout than OPD. As shown in Figure~\ref{fig:process_memory}, for OPD, the efficient tool call requires a large concurrency (32 processes), and the deployment of environment sometimes requires GPU (80GB memory for search), while ReOPD doesn't need these at all. It is also predictable that the required resources for OPD will increase  for the multi-environment distillation, but ReOPD doesn't need any online environment. Thus, ReOPD preserves or improves OPD-level accuracy while replacing expensive online interaction with reusable offline teacher trajectories.

\subsection{Analysis and Discussion}
\paragraph{Validating the step-level schedule.}
Figure~\ref{fig:poc_hints} tests whether the depth proxy of Section~\ref{sec:aopd}
actually matters for final performance. Because the prefix is teacher-forced and only the current step is
student-on-policy, the two-sided shift is smallest at early steps and grows with
depth. Both panels use the sampling implementation: varying which chunk is drawn
(Fig.~\ref{fig:chunk_index}) performs best when it draws more early chunks, and varying
the decay steepness \(\kappa\) (Fig.~\ref{fig:sampling_decay}) shows the same trend as
late steps are sampled less. The agreement between the two knobs corroborates the
step-decay schedule of Section~\ref{sec:aopd}; the improvement from favoring early,
low-shift steps shows that the gain comes from reliability-aware prefix selection
rather than simply using more data.

\paragraph{Two regimes in practice.}
The tables reveal the two regimes anticipated by our analysis. On search and
multi-hop QA, where the teachers stay reliable on student-induced histories, the
student-occupancy shift dominates: student-on-policy OPD is already strong, and
ReOPD -- whose teacher-anchored roll-in is already student-relevant here, since the two
occupancies nearly coincide -- essentially matches it (e.g.\ a $40.5$ vs.\ $40.6$
average under the single-environment setting). On mathematical reasoning the picture changes as
the teacher--student gap widens. With stronger teachers (Qwen3-8B and
Qwen3-30B-A3B), student roll-ins drift into histories where the teacher's trajectory
is the more trustworthy signal, so the teacher-reliability shift dominates; ReOPD's
teacher-anchored roll-in sidesteps these unreliable histories and attains its largest
gains, improving AIME24 from $28.3$ to $36.7$ under the Qwen3-8B teacher while
lifting the 4B-student math average in every teacher setting ($55.1\!\to\!57.2$,
$51.0\!\to\!53.7$, $51.1\!\to\!52.5$). Which roll-in is preferable thus depends on
the regime -- the teacher-anchored pool wins when the gap is large and ties
student-on-policy OPD when the teacher is already reliable, rather than committing to a
fixed student-on-policy roll-in. Our analysis prescribes a steeper decay as the
teacher--student gap grows (Section~\ref{sec:aopd}); notably, these results obtain the
regime-appropriate behavior with a single fixed \(\kappa\!=\!0.6\) across tasks rather
than per-task tuning, and gap-adaptive schedules are a natural extension.

\begin{figure}[t]
    \centering
    \includegraphics[width=\linewidth]{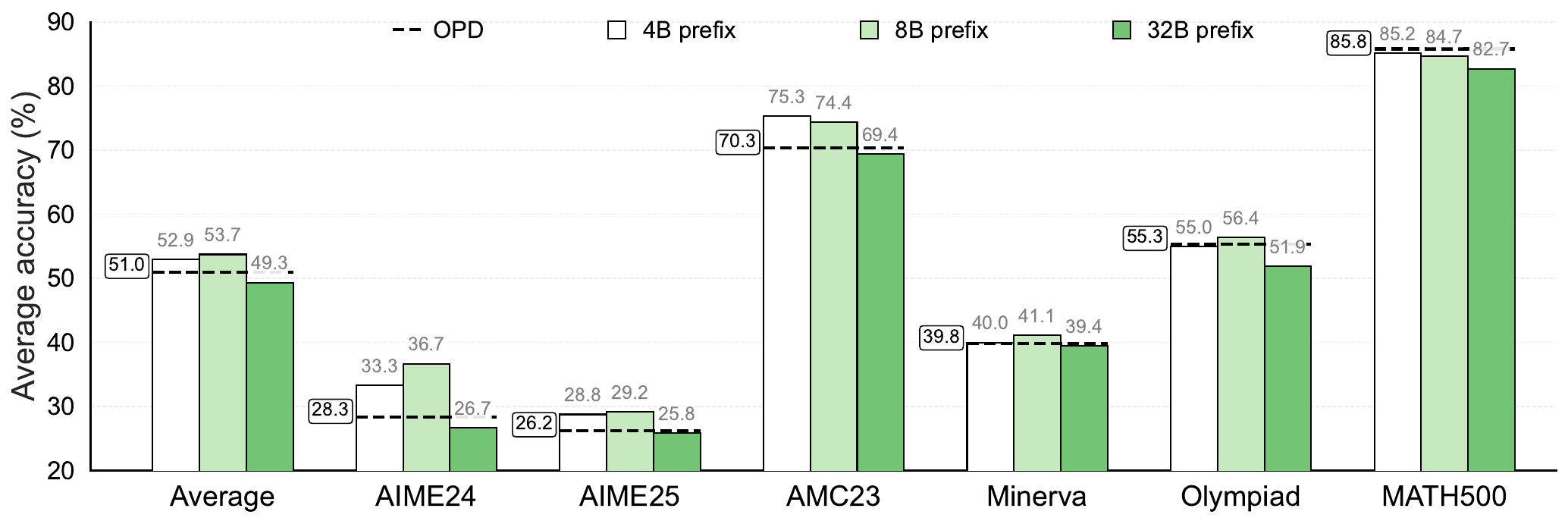}
 %   \vspace{-6mm}
  \caption[]{\textbf{Prefix source: the teacher's own model gives the best prefixes,
  not a stronger generator.} Teacher Qwen3-8B, student Qwen3-4B-Instruct-2507; we fix
  the distillation target to the teacher and vary only the prefix generator.
  Performance peaks when prefixes come from the teacher itself and \emph{degrades} with
  larger or stronger generators -- a direct test of the teacher-reliability view: the
  teacher's conditional is trustworthy only on prefixes within its own support.}
  \label{fig:prefix_source}
%  \vspace{-4mm}
\end{figure}

\paragraph{Prefix source: reliability, not raw capability.}
Step index is only a proxy; what it ultimately stands in for is support overlap
between the student-relevant histories and the teacher's reliable region.
Figure~\ref{fig:prefix_source} isolates this by varying \emph{where prefixes come
from} while holding the distillation target fixed at the teacher. One might expect prefixes
from a larger or stronger generator to help, since such models make fewer mistakes.
We observe the opposite: prefixes drawn from the \emph{teacher's own model} are
best, and substituting a larger or stronger prefix generator degrades the distilled
student. This is a direct test of the teacher-reliability shift. The teacher's
recorded conditional is a reliable improvement signal only on the histories the teacher
itself is likely to visit; prefixes from a different, stronger model fall outside
this support, so querying the teacher there yields a less trustworthy target. The
right notion for prefix selection is therefore \emph{reliability with respect to the
teacher}, not the standalone capability of the prefix generator -- precisely the
quantity ReOPD's reliability-aware prefix selection is designed to respect.

\begin{table}[t]
    \caption{\textbf{RL-collected pool vs.\ stationary teacher pool.} ReOPD reuses the
    teacher's RL rollouts, a \emph{mixed-policy} pool spanning early-to-late checkpoints,
    rather than a stationary pool drawn from the final teacher. Distilling against the
    final teacher from each pool gives nearly identical students, so the free RL
    by-product is as good as a dedicated collection. The teacher and student models are
    Qwen3-8B and Qwen3-4B-Instruct-2507, respectively.}
    \label{tab:pool_source}
    \small
    \centering
    \begin{tabular}{lccccccc}
        \toprule
            \textbf{Prefix pool} & \textbf{AIME24} & \textbf{AIME25} & \textbf{AMC23} & \textbf{Minerva} & \textbf{Olympiad} & \textbf{MATH500} & \textbf{\textit{Avg.}} \\
        \midrule
            Final-teacher (stationary) & 37.9 & 30.0 & 72.2 & 39.2 & 56.2 & 85.0 & 53.4  \\
            \rowcolor{myblue}
            RL rollouts (free, mixed) & 36.7 & 29.2 & 74.4 & 41.1 & 56.4 & 84.7 & 53.7 \\
        \bottomrule
    \end{tabular}
\end{table}

\paragraph{RL-collected pool matches a stationary teacher pool.}
A natural concern is that the RL-collected pool is \emph{not} stationary: it mixes
rollouts from weaker early checkpoints with the final teacher, so it need not match a
pool drawn purely from the converged teacher. Table~\ref{tab:pool_source} tests this
by fixing the distillation target to the final teacher and varying only the prefix
pool. The two pools yield nearly identical students, confirming that the
checkpoint-to-checkpoint drift in RL rollouts is within the teacher's reliable support
and that no dedicated collection is needed -- the free by-product suffices.

\paragraph{Teacher's prefix resampling.}
For some cases, we don't have the teacher's prefix from RL. For example, the teacher is already a strong open-sourced model. And we don't have resource to do RL on it. We can deploy the environment for the teacher to sample all trajectories for all prompts at once, and use these trajectories as prefixes for ReOPD. Table \ref{tab:pool_source} already shows that the performance is similar to the one using prefixes from RL. In Table \ref{fig:time_resample}, we take the time used for the teacher's resampling into account when comparing to OPD. Our ReOPD is still efficient, $> 2\times$. Notably, the environment is only online for the teacher's resampling, and deactivated for the training of student, which is resource efficient.

\begin{wrapfigure}{r}{0.43\textwidth}
    \vspace{-6mm}
    \includegraphics[width=0.98\linewidth]{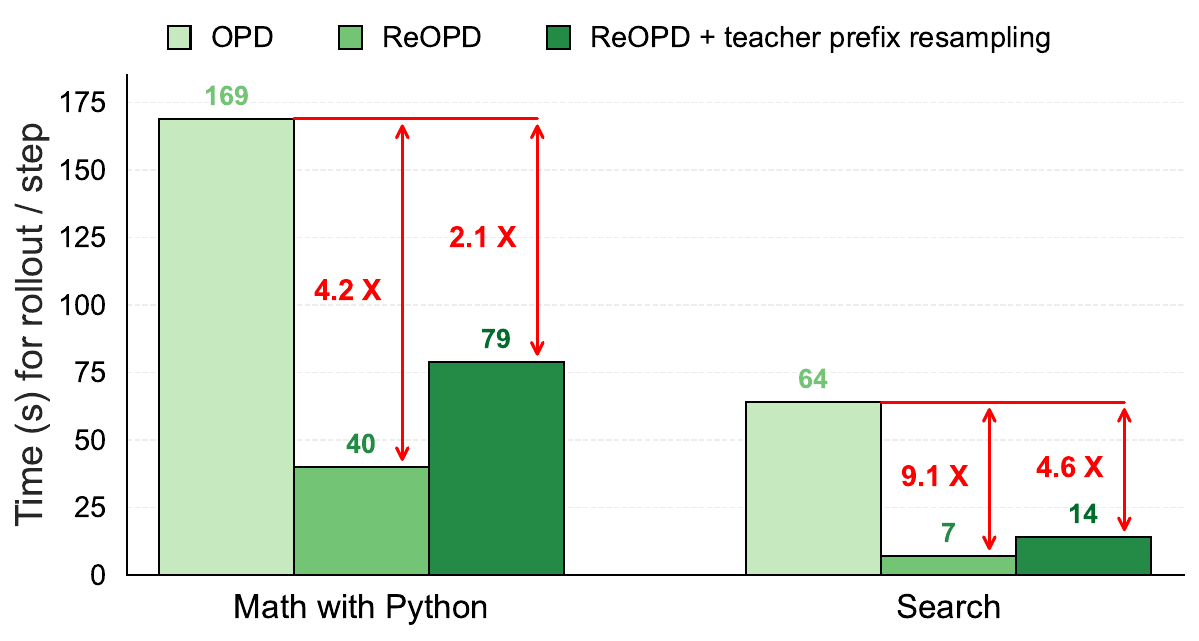} 
    \vspace{-3mm}
    \caption[]{\textbf{Rollout time when taking the teacher's rollout into account.} When the teacher's trajectories from RL are not available, one could use the teacher to sample the trajectories for all prompts at once. We take this time into account for the comparison. For this setting, the environment is only online for the teacher's rollout.}
    \label{fig:time_resample}
    \vspace{-6mm}
\end{wrapfigure}

\FloatBarrier

%% file: sections/conclusion.tex
\section{Conclusion}

We studied Replayed-Prefix On-Policy Distillation in an efficient, off-environment regime
that reuses a fixed pool of pre-collected teacher trajectories, replaying a
teacher-forced prefix and supervising the student's action at the current step.
Removing the environment does not remove the temporal structure of multi-turn learning:
we identified a \emph{prefix trap} with a temporal layer (compounding errors) and a
distributional layer (a two-sided shift), and bounded the gap to an ideal interactive
improvement objective~(\plaineqref{eq:ideal}) by a student occupancy-mismatch term and a
teacher reliability term~(\plaineqref{eq:bound}). This recasts multi-turn OPD as \emph{reliability-aware prefix distribution design}
-- balancing student relevance against teacher reliability, with fully
student-on-policy and teacher-forced roll-ins as the two extremes -- which the
resulting algorithm, ReOPD, realizes with a simple step-decay schedule applied by
default as a sampling probability (equivalently, a loss weight). Our experiments confirm the predicted
regime-aware behavior: ReOPD stays OPD-like when the teacher--student gap is small and
yields its largest gains when the gap is large and teacher-supported prefixes are the
more reliable supervision region.

\paragraph{Limitations and future work.}
Our analysis assumes access to a pre-collected pool of teacher trajectories with
recorded observations. In our experiments this comes for free from the teacher's
own RL rollouts, but more generally the pool's coverage and quality bound what the
student can learn. The reliability surrogate we adopt reduces, to first order, to a
step-decaying weight, which is effective but coarse: it captures \emph{how deep} a
prefix is rather than directly measuring \emph{how far} a given prefix lies from the
overlap between student-relevant histories and the teacher's reliable support, and a
learned or data-dependent reliability estimate may improve on the schedule. Finally,
our two-sided bound treats the teacher reliability error through teacher support as a
surrogate; tighter, directly estimable notions of teacher reliability -- and schedules
or weights that adapt to them online -- are a natural next step.